\documentclass[sigconf]{acmart}
\AtBeginDocument{%
  \providecommand\BibTeX{{%
    \normalfont B\kern-0.5em{\scshape i\kern-0.25em b}\kern-0.8em\TeX}}}

\setcopyright{acmcopyright}

\copyrightyear{2020}
\acmYear{2020}
\setcopyright{iw3c2w3}
\acmConference[WWW '20]{Proceedings of The Web Conference 2020}{April 20--24, 2020}{Taipei, Taiwan}
\acmBooktitle{Proceedings of The Web Conference 2020 (WWW '20), April 20--24, 2020, Taipei, Taiwan}
\acmPrice{}
\acmDOI{10.1145/3366423.3380112}
\acmISBN{978-1-4503-7023-3/20/04}

\usepackage{amssymb}
\usepackage{amsmath}
\usepackage{multirow}
\usepackage{graphicx}
\usepackage{mathtools}
\usepackage{amsthm}
\usepackage{threeparttable}
\usepackage{booktabs}
\usepackage{balance}

\newtheorem{theorem}{Theorem}
\newtheorem{lemma}{Lemma}
\newtheorem{definition}{Definition}

\def\Mat#1{{\boldsymbol{#1}}}
\def\Vec#1{{\boldsymbol{#1}}}

\begin{document}

\title{Graph Representation Learning via Graphical Mutual Information Maximization}

\author{Zhen Peng$^{1*}$, Wenbing Huang$^{2\dagger}$, Minnan Luo$^{1\dagger}$, Qinghua Zheng$^1$, Yu Rong$^3$,}
\author{Tingyang Xu$^3$, Junzhou Huang$^3$}
\thanks{$^{\dagger}$Corresponding authors.\\ $^{*}$Work done during an internship at Tencent AI Lab.}
\affiliation{
	\institution{$^1$Ministry of Education Key Lab for Intelligent Networks and Network Security, School of Computer Science and Technology, Xi'an Jiaotong University, China}
	\institution{$^2$Beijing National Research Center for Information Science and Technology (BNRist), State Key Lab of Intelligent Technology and Systems, Department of Computer Science and Technology, Tsinghua University, China}
	\institution{$^3$Tencent AI Lab, China}
}
\email{zhenpeng27@outlook.com, hwenbing@126.com, {minnluo,qhzheng}@xjtu.edu.cn,} 
\email{yu.rong@hotmail.com, tingyangxu@tencent.com, jzhuang@uta.edu}

\fancyhead{}
\renewcommand{\shortauthors}{Z. Peng, W. Huang, and M. Luo, et al.}

\begin{abstract}
  The richness in the content of various information networks such as social networks and communication networks provides the unprecedented potential for learning high-quality expressive representations without external supervision. This paper investigates how to preserve and extract the abundant information from graph-structured data into embedding space in an unsupervised manner. To this end, we propose a novel concept, Graphical Mutual Information (GMI), to measure the correlation between input graphs and high-level hidden representations. GMI generalizes the idea of conventional mutual information computations from vector space to the graph domain where measuring mutual information from two aspects of node features and topological structure is indispensable. GMI exhibits several benefits: First, it is invariant to the isomorphic transformation of input graphs---an inevitable constraint in many existing graph representation learning algorithms; Besides, it can be efficiently estimated and maximized by current mutual information estimation methods such as MINE; Finally, our theoretical analysis confirms its correctness and rationality. With the aid of GMI, we develop an unsupervised learning model trained by maximizing GMI between the input and output of a graph neural encoder. Considerable experiments on transductive as well as inductive node classification and link prediction demonstrate that our method outperforms state-of-the-art unsupervised counterparts, and even sometimes exceeds the performance of supervised ones.
\end{abstract}

\ccsdesc[500]{Mathematics of computing~Graph algorithms}
\ccsdesc[500]{Computing methodologies~Neural networks}

\keywords{Graph representation learning, Mutual information, InfoMax.}


\maketitle

\section{Introduction}
To achieve reliable quantitative analysis for diverse information networks, high-quality representation learning for graph-structured data has become one of the current fascinating topics. Recent methods towards this goal, commonly categorized as Graph Neural Networks (GNNs), have made remarkable advancements in a great many learning tasks, such as node classification~\cite{zhang2019bayesian,wu2019net,qu2019gmnn}, link prediction~\cite{liben2007link,kipf2016variational,zhang2018link}, graph alignment~\cite{heimann2018regal,Faerman2019GraphAN,wu2019relation}, molecular generation~\cite{bjerrum2017molecular,you2018graph,bresson2019two}, to name a few. Albeit in fruitful progress, training GNNs in existing approaches usually requires a certain form of supervision. Undoubtedly, the labeling information is expensive to acquire---manual annotation or paying for permission, and is even impossible to attain because of the privacy policy. Not to mention that the reliability of given labels is sometimes questionable. Hence, how to achieve high-quality graph representation without supervision becomes necessitated for a great many practical cases, which motivates the study of this paper.

\begin{figure*}[t] 
	\centering 
	\includegraphics[width=0.9\textwidth]{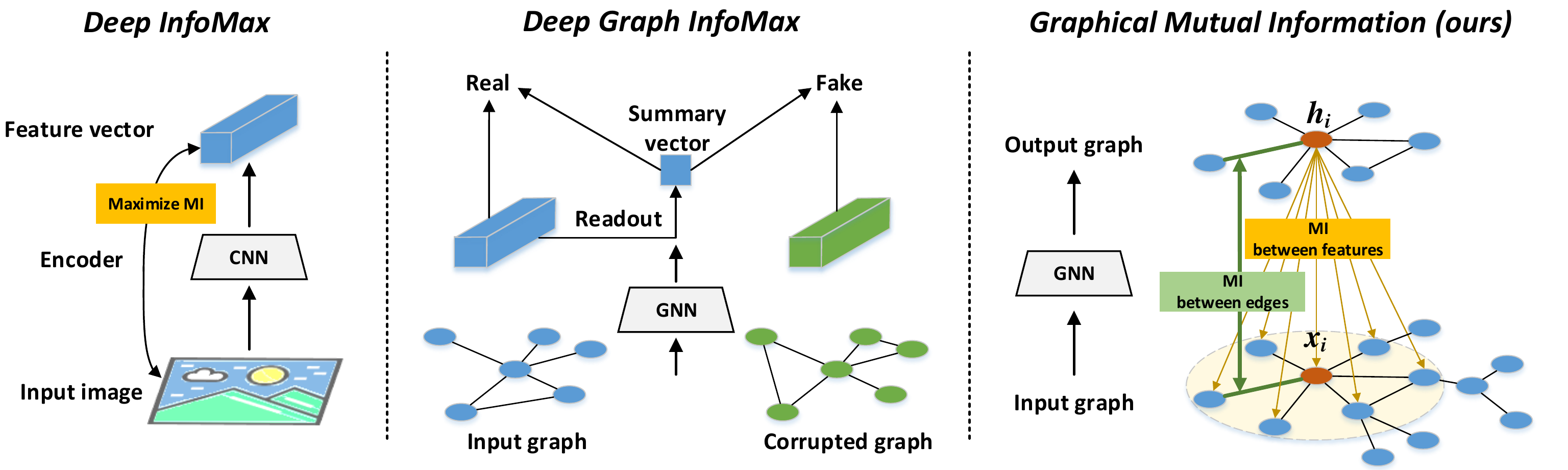}
	\caption{A high-level overview of Deep InfoMax (\textbf{left}), Deep Graph InfoMax (DGI)\  (\textbf{middle}), and Graphical Mutual Information (GMI)\  (\textbf{right}). Note that graphs with topology and features are more complicated than images that involve features only, thus GMI ought to maximize the MI of both features and edges between inputs (\emph{i.e.}, an input graph) and outputs (\emph{i.e.}, an output graph) of the encoder. The architecture of GMI is quite different from that of DGI which will be explained in the introduction.}
	\label{fig:graph4}
\end{figure*}

Through carefully examining our collected graph-structured data, we find that they generally come from various sources such as social networks, citation networks, and communication networks where a tremendous amount of both content and linkage information exist. For instance, data on many social platforms like Twitter, Flickr, and Facebook include features of users, \emph{e.g.}, basic personal details, texts, images, IP, and their relations, \emph{e.g.}, buying the same item, being friends. These rich content data are sufficient to support subsequent mining tasks without additional guidance: if two entities exhibit the extreme similarity in features, there is a high probability of a link between them (link prediction), and they are likely to belong to the same category (classification); if two entities both link to the same entity, they probably have similar characteristics (recommendation). In this sense, preserving and extracting as much information as possible from information networks into embedding space facilitates learning high-quality expressive representations that exhibit desirable performance in mining tasks without any form of supervision. Unsupervised graph representation learning is a more favorable choice in many cases due to the freedom from labels, particularly when we intend to take benefit from a large scale unlabeled data in the wild.

To fully inherit the rich information in graphs, in this paper, we execute graph embedding based upon Mutual Information (MI) maximization, inspired by the empirical success of the Deep InfoMax method~\cite{hjelm2018learning} which operates on images. To discover useful representations, Deep InfoMax trains the encoder to maximize MI between its inputs (\emph{i.e.}, the images) and outputs (\emph{i.e.}, the hidden vectors). When considering Deep InfoMax in the graph domain, the first stone we need to step over is how to define MI between graphs and hidden vectors, whereas the topology of graphs is more complicated than images (see Figure~\ref{fig:graph4}). One of the challenges is to ensure the MI function between each node's hidden representation and its neighborhood input features to obey the symmetric property, or equivalently, being invariant to permutations of the neighborhoods. As one recent work considering MI, Deep Graph Infomax (DGI)~\cite{velivckovic2018deep} first embeds a input graph and a corresponding corrupted graph, then summarizes the input graph as a vector via a readout function, finally maximizes MI between this summary vector and hidden representations by discriminating the input graph (positive sample) from the corrupted graph (negative sample). Figure~\ref{fig:graph4} gives an easily understandable overview of DGI. Maximizing this kind of MI is proved to be equivalent to maximizing the one between the input node features and hidden vectors, but this equivalence holds under several preset conditions, \emph{e.g.}, the readout function should be injective, which yet seem to be over-restricted in real cases. Even we can guarantee the existence of injective readout function by certain design, \emph{e.g.}, the one used in DeepSets~\cite{zaheer2017deep}, the injective ability of readout function is also affected by how its parameters are trained. That is to say that an originally-injective function still has the risk of becoming non-injective if it is trained without any external supervision. And if the readout function is not injective, the input graph information contained in a summary vector will diminish as the size of the graph increases. Moreover, DGI stays in a coarse graph/patch-level MI maximization. Hence in DGI, there is no guarantee that the encoder can distill sufficient information from input data as it never elaborately correlates hidden representations with their original inputs.

In this paper, we put forward a more straightforward way to consider MI in terms of graphical structures without using any readout function and corruption function. We directly derive MI by comparing the input (\emph{i.e.}, the sub-graph consisting of the input neighborhood) and the output (\emph{i.e.}, the hidden representation of each node) of the encoder. And interestingly, our theoretical derivations demonstrate that the directly-formulated MI can be decomposed into a weighted sum of local MIs between each neighborhood feature and the hidden vector. In this way, we have decomposed the input features and made the MI computation tractable. Moreover, this form of MI can easily satisfy the symmetric property if we adjust the values of weights. We defer more details to~\textsection~\ref{Sec:FMI}. 
As the above MI is mainly measured at the level of node features, we term it as Feature Mutual Information (FMI). 

Two remaining issues about FMI: \textbf{1.} the combining weights are still unknown and \textbf{2.} it does not take the topology (\emph{i.e.}, the edge features) into account. To further address these two issues, we define our Graphical Mutual Information (GMI) measurement based on FMI. In particular, GMI applies an intuitive value assignment by setting the weights in FMI equal to the proximity between each neighbor and the target node in the representation space. As to retain the topology information, GMI further correlates these weights with the input edge features via an additional mutual information term. The resulting GMI is topologically invariant and also calculable with Mutual Information Neural Estimation (MINE)~\cite{belghazi2018mine}. The main contributions of our work are as follows:
\begin{itemize}
	\item \textit{\textbf{Concepts}}: We generalize the conventional MI estimation to the graph domain and propose a new concept of Graphical Mutual Information (GMI) accordingly. GMI is free from the potential risk caused by the readout function since it considers MI between input graphs and high-level embeddings in a straightforward pattern. 
	\item \textit{\textbf{Algorithms}}: Through our theoretical analysis, we give a tractable and calculable form of GMI which decomposes the entire GMI into a weighted sum of local MIs. With the help of the MINE method, GMI maximization can be easily achieved in a node-level.
	\item \textit{\textbf{Experimental Findings}}: We verify the effectiveness of GMI on several popular node classification and link prediction tasks including both transductive and inductive ones. The experiments demonstrate that our method delivers promising performance on a variety of benchmarks and it even sometimes outperforms the supervised counterparts. 
\end{itemize}

\section{Related Work}
In line with the focus of our work, we briefly review the previous work in the two following areas: \textbf{1.} mutual information estimation, and \textbf{2.} neural networks for learning representation over graphs.

\textbf{Mutual information estimation}. As InfoMax principle~\cite{bell1995information} advocates maximizing MI between the inputs and outputs of neural networks, many methods such as ICA algorithms~\cite{hyvarinen1999nonlinear,almeida2003misep} attempt to employ the idea of MI in unsupervised feature learning. Nonetheless, these methods can not be generalized to deep neural networks easily due to the difficulty in calculating MI between high dimensional continuous variables. Fortunately, Mutual Information Neural Estimation (MINE)~\cite{belghazi2018mine} makes the estimation of MI on deep neural networks feasible via training a statistics network as a classifier to distinguish samples coming from the joint distribution and the product of marginals of two random variables. Specifically, MINE uses the exact KL-based formulation of MI, while a non-KL alternative, the Jensen-Shannon divergence (JSD)~\cite{nowozin2016f}, can be used without the concern about the precise value of MI. 

\textbf{Neural networks for graph representation learning}. With the rapid development of graph neural networks (GNNs), a large number of graph representation learning algorithms based on GNNs are proposed in recent years, which exhibit stronger performance than traditional random walk-based and factorization-based embedding approaches~\cite{perozzi2014deepwalk,tang2015line,cao2015grarep,grover2016node2vec,qiu2018network}. Typically, these methods can be divided into supervised and unsupervised categories. Among them, there is a rich literature on supervised representation learning over graphs~\cite{kipf2016semi,velivckovic2017graph,chen2018fastgcn,zhang2018gaan,ding2018semi}. In spite of their variance in network architecture, they achieve empirical success with the help of labels that are often not accessible in realistic scenarios. In this case, unsupervised graph learning methods~\cite{hamilton2017inductive,duran2017learning,velivckovic2018deep} have broader application potential. The well-known method is GraphSAGE~\cite{hamilton2017inductive}, an inductive framework to train GNNs by a random-walk based objective in its unsupervised setting. And recently, DGI~\cite{velivckovic2018deep} applies the idea of MI maximization to the graph domain and obtains the strong performance in an unsupervised pattern. However, DGI implements a coarse-grained maximization ($i.e.$, maximizing MI at graph/patch-level) which makes it difficult to preserve the delicate information in the input graph. Besides, the condition imposed on the readout function used in DGI seems to be over-restricted in real cases. By contrast, we focus on removing out the restriction of readout function and arriving at graphical mutual information maximization in a node-level by directly maximizing MI between inputs and outputs of the encoder. Representations derived by our method are more sophisticated in keeping input graph information, which ensures its potential for downstream graph mining tasks, \emph{e.g.}, node classification, link prediction, and recommendation. 

\section{Graphical Mutual Information: Definition and Maximization}

Prior to going further, we first provide the preliminary concepts used in this paper. Let $\mathcal{G}=(\mathcal{V}, \mathcal{E})$ denote a graph with $N$ nodes $v_i\in\mathcal{V}$ and edges $e_{ij}=(v_i, v_j)\in\mathcal{E}$. The node features, with assumed empirical probability distribution $\mathbb{P}$, are given by $\Mat{X}\in\mathbb{R}^{N\times D}=\{\Vec{x}_1, \cdots, \Vec{x}_N\}$ where $\Vec{x}_i\in\mathbb{R}^{D}$ denotes the feature for node $v_i$. 
The adjacency matrix $\Mat{A}\in\mathbb{R}^{N\times N}$ represent edge connections, where $A_{ij}$ associated to edge $e_{ij}$ could be a real number or multi-dimensional vector\footnote{Our method is generally applicable to the graphs with edge features, although we only consider edges with real weights in our experiments.}. 

The goal of graph representation learning is to learn an encoder $f: \mathbb{R}^{N\times D}\times \mathbb{R}^{N\times N} \rightarrow \mathbb{R}^{N\times D'}$, such that the hidden vectors $\Mat{H}=\{\Vec{h}_1, \cdots, \Vec{h}_N\}=f(\Mat{X}, \Mat{A})$ indicate high-level representations for all nodes. The encoding process can be rewritten in a node-wise form. To show this, we define $\Mat{X}_i$ and $\Mat{A}_i$ for node $i$ as respectively the features of its neighbors and the corresponding adjacency matrix conditional on the neighbors. Particularly, $\Mat{X}_i$ consists of all k-hop neighbors of $v_i$ with $k\leq l$ when the encoder $f$ is an $l$-layer GNN, and it contains the node $i$ itself if we further add self-loops in the adjacency matrix. 
Here, we call the sub-graph expanded by $\Mat{X}_i$ and $\Mat{A}_i$ as a \emph{support graph} for node $i$, denoted by $\mathcal{G}_i$. With the definition of support graph, the encoding for each node becomes $\Vec{h}_i = f(\mathcal{G}_i)=f(\Mat{X}_i, \Mat{A}_i)$. 

\textbf{Difficulties in defining graphical mutual information.}
In Deep InfoMax~\cite{hjelm2018learning}, the training objective of the encoder is to maximize MI between its inputs and outputs. The MI is estimated by employing a statistics network as a discriminator to classify samples coming from the joint distribution and the ones drawn from the product of marginals. Naturally, when adapting the idea of Deep InfoMax to graphs, we should maximize MI between the representation $\Vec{h}_i$ and the support graph $\mathcal{G}_i$ for each node. We denote such graphical MI as $I(\Vec{h}_i; \mathcal{G}_i)$. However, it is non-straightforward to define $I(\Vec{h}_i; \mathcal{G}_i)$. The difficulties are: 
\begin{itemize}
	\item The graphical MI should be invariant concerning the node index. In other words, we should have $I(\Vec{h}_i; \mathcal{G}_i)=I(\Vec{h}_i; \mathcal{G'}_i)$, if $\mathcal{G}_i$ and $\mathcal{G'}_i$ are isomorphic to each other.
	\item If we adopt MINE method for MI calculation, the discriminator in MINE only accepts inputs of a fixed size. This yet is infeasible for $\mathcal{G}_i$ as different $\mathcal{G}_i$ usually include different numbers of nodes and thus are of distinct sizes.
\end{itemize}

To get around the issue of defining graphical mutual information, this section begins with introducing the concept of Feature Mutual Information (FMI) that only relies on node features. Upon the inspiration from the decomposition of FMI, we then define Graphical Mutual Information (GMI), which takes both the node features and graph topology into consideration.

\subsection{Feature Mutual Information}\label{Sec:FMI}
We denote the empirical probability distribution of node features $\Mat{X}_i$ as $p(\Mat{X}_i)$, the probability of $\Vec{h}_i$ as $p(\Vec{h}_i)$, and the joint distribution by $p(\Vec{h}_i, \Mat{X}_i)$. According to the information theory, the MI between $\Vec{h}_i$ and $\Mat{X}_i$ is defined as
\begin{eqnarray}
\label{Eq:FMI}
I(\Vec{h}_i; \Mat{X}_i) \ =\  \int_{\mathcal{H}}\int_{\mathcal{X}} p(\Vec{h}_i, \Mat{X}_i)\log\frac{ p(\Vec{h}_i, \Mat{X}_i)}{ p(\Vec{h}_i)p(\Mat{X}_i)}d\Vec{h}_i d\Mat{X}_i.
\end{eqnarray}
Interestingly, we have the following mutual information decomposition theorem for computing $I(\Vec{h}_i; \Mat{X}_i)$. 
\begin{theorem} [Mutual Information Decomposition]
	\label{Th:FMI}
	If the conditional probability $p(\Vec{h}_i|\Mat{X}_i)$ is multiplicative (see the definition of multiplicative in~\cite{renner2002mutual}), 
	the global mutual information $I(\Vec{h}_i; \Mat{X}_i)$ defined in Eq.~\eqref{Eq:FMI} can be decomposed as a weighted sum of local MIs, namely,
	\begin{eqnarray}
	\label{Eq:decompose}
	I(\Vec{h}_i; \Mat{X}_i) &=& \sum\nolimits_{j}^{i_n} w_{ij} I(\Vec{h}_i; \Mat{x}_{j}),
	\end{eqnarray}
	where, $\Mat{x}_{j}$ is the $j$-th neighbor of node $i$,  $i_n$ is the number of all elements in $\Mat{X}_i$, and the weight $w_{ij}$ satisfies $\frac{1}{i_n}\leq w_{ij} \leq 1$ for each $j$.
\end{theorem}

To prove the above theorem, we first introduce two lemmas and a definition.
\begin{lemma}\label{lem1}
	For any random variables $X$, $Y$, and $Z$, we have 
	\begin{eqnarray}
	I(X,Y; Z) &\geq& I(X; Z).
	\end{eqnarray}
\end{lemma}
\begin{proof}
	\begin{equation}\nonumber
	\begin{aligned}
	&\quad\  I(X,Y; Z)-I(X; Z)\\
	&= \iiint_{\mathcal{X}\mathcal{Y}\mathcal{Z}} p(X,Y,Z)\log\frac{p(X,Y,Z)}{p(X,Y)p(Z)} dXdYdZ - \\ &\quad  \iint_{\mathcal{X}\mathcal{Z}}p(X,Z)\log\frac{p(X,Z)}{p(X)p(Z)}dXdZ\\
	&= \iiint_{\mathcal{X}\mathcal{Y}\mathcal{Z}} p(X,Y,Z)\log\frac{ p(X,Y,Z)}{ p(X,Y)p(Z)} dXdYdZ - \\ &\quad \iiint_{\mathcal{X}\mathcal{Y}\mathcal{Z}}p(X,Y,Z)\log\frac{p(X,Z)}{p(X)p(Z)} dXdYdZ\\
	&= \iiint_{\mathcal{X}\mathcal{Y}\mathcal{Z}}p(X,Y,Z)\log\frac{p(X,Y,Z)}{p(X,Y)} \cdot \frac{p(X)}{p(X,Z)} dXdYdZ\\
	&=\iiint_{\mathcal{X}\mathcal{Y}\mathcal{Z}}p(X,Y,Z)\log\frac{p(X,Y,Z)}{p(Y|X)p(X,Z)} dXdYdZ\\
	&=\iiint_{\mathcal{X}\mathcal{Y}\mathcal{Z}}p(Y,Z|X)p(X)\log\frac{p(Y,Z|X)}{p(Y|X)p(Z|X)} dXdYdZ\\
	&=I(Y; Z|X)\ \geq\ 0
	\end{aligned}
	\end{equation}
	Thus we achieve $I(X,Y; Z) \geq I(X; Z)$.
\end{proof}

\begin{definition}
	The conditional probability $p(h|X_{1},\cdots,X_{n})$ is called multiplicative if it can be written as a product
	\begin{eqnarray}
	p(h|X_{1},\cdots,X_{n}) &=& r_{1}(h,X_{1})\cdots r_{n}(h,X_{n}) 
	\end{eqnarray}
	with appropriate functions $r_{1},\cdots,r_{n}$.
\end{definition}

\begin{lemma}\label{lem2}
	If $p(h|X_{1},\cdots,X_{n})$ is multiplicative, then we have 
	\begin{eqnarray}
	I(X;Z)\ +\ I(Y;Z) &\geq& I(X,Y;Z)
	\end{eqnarray}
\end{lemma}
\begin{proof}
	See~\cite{renner2002mutual} for detailed proof.
\end{proof}

\noindent Now all the necessities for proving Theorem~\ref{Th:FMI} are in place.
\begin{proof}
	According to Lemma~\ref{lem1}, for any $j$ we have 
	\begin{eqnarray}
	I(\Vec{h}_{i};\Mat{X}_{i}) &=& I(\Vec{h}_{i};\Mat{x}_{1},\cdots,\Mat{x}_{i_{n}})\ \ \geq\ \  I(\Vec{h}_{i};\Mat{x}_{j}).
	\end{eqnarray}
	It means 
	\begin{eqnarray}\label{refq1}
	I(\Vec{h}_{i};\Mat{X}_{i}) &=& \sum\frac{1}{i_{n}}I(\Vec{h}_{i};\Mat{X}_{i})\ \geq\  \sum\frac{1}{i_{n}}I(\Vec{h}_{i};\Mat{x}_{j}).
	\end{eqnarray}
	On the other hand, based on Lemma~\ref{lem2}, we get
	\begin{eqnarray}\label{refq2}
	I(\Vec{h}_{i};\Mat{X}_{i}) &\leq& \sum I(\Vec{h}_{i};\Mat{x}_{j}).
	\end{eqnarray}
	Then the above two formulas could deduce the following
	\begin{eqnarray}\label{eqref1}
	\sum\frac{1}{i_{n}}I(\Vec{h}_{i};\Mat{x}_{j}) &\leq& I(\Vec{h}_{i};\Mat{X}_{i})\ \ \leq\ \ \sum I(\Vec{h}_{i};\Mat{x}_{j}).
	\end{eqnarray}
	As all $I(\Vec{h}_{i};\Mat{x}_{j})\ \geq\ 0$, there must exist weights $\frac{1}{i_{n}}\ \leq\ w_{ij}\ \leq\ 1$. When setting $w_{ij}=I(\Vec{h}_i; \Mat{X}_i)\, /\, \sum I(\Vec{h}_i; \Vec{x}_j)$, we will achieve Eq.~\eqref{Eq:decompose} while ensuring $1/i_{n}\,\leq\,w_{ij}\,\leq\,1$ by Eq.~\eqref{eqref1}, thus the Theorem~\ref{Th:FMI} has been proved.
\end{proof}

With the decomposition in Theorem~\ref{Th:FMI}, we can calculate the right side of Eq.~\eqref{Eq:decompose} via MINE as inputs of the discriminator now become the pairs of $(\Vec{h}_i, \Mat{x}_{j})$ whose size always keep the same (\emph{i.e.}, $D'$-by-$D$). Besides, we can adjust the weights to reflect the isomorphic transformation of input graphs. For instance, if $\Mat{X}_i$ only contains one-hop neighbors of node $i$, setting all weights to be identical will lead to the same MI for the input nodes in different orders. 

Despite some benefits of the decomposition, it is hard to characterize the exact values of the weights since they are related to the values of $I(\Vec{h}_i; \Mat{x}_j)$ and their underlying probability distributions. A trivial way is setting all weights to be $\frac{1}{i_n}$, then maximizing the right side of Eq.~\eqref{Eq:decompose} equivalents to maximizing the lower bound of $I(\Vec{h}_i; \Mat{X}_i)$, by which the true FMI is also maximized to some extent. Besides this method, we additionally provide a more enhanced solution by considering the weights as trainable attentions, which is the topic in the next subsection.

\begin{table*}[t]
	\centering
	\begin{threeparttable}
		\caption{Statistics of the datasets used in experiments.}\label{table:t1}  
		\begin{tabular}{cccccccc}
			\toprule
			\multicolumn{2}{c}{\textbf{Task}}& \textbf{Dataset} & \textbf{Type}& \textbf{$\#$Nodes} & \textbf{$\#$Edges} & \textbf{$\#$Features} & \textbf{$\#$Classes} \\
			\midrule
			\multirow{5}*{Classification} &\multirow{3}{*}{Transductive}& Cora & Citation network& 2,708 & 5,429 & 1,433 & 7\\
			&& Citeseer & Citation network& 3,327 & 4,732 & 3,703 & 6\\
			&& PubMed & Citation network& 19,717 & 44,338 & 500 & 3\\
			\cmidrule(lr){2-8} 
			&\multirow{2}{*}{Inductive}& Reddit & Social network& 232,965 &  11,606,919 & 602 & 41\\
			&& PPI & Protein network& 56,944 & 806,174 & 50 & 121$^{*}$\\
			\midrule
			\multicolumn{2}{c}{\multirow{4}*{Link prediction}} & Cora & Citation network& 2,708 & 5,429 & 1,433 & 7\\
			&&BlogCatalog& Social network& 5,196& 171,743& 8,189&6\\
			&&Flickr& Social network& 7,575& 239,738& 12,047&9\\
			&&PPI& Protein network& 56,944 & 806,174 & 50 & 121\\
			\bottomrule
		\end{tabular}
		\begin{tablenotes}
			\item[$*$] The task on PPI is a multilabel classification problem.
		\end{tablenotes}
	\end{threeparttable}
\end{table*}

\subsection{Topology-Aware Mutual Information}
\label{Sec:GMI}
Inspired from the decomposition in Theorem~\ref{Th:FMI}, we attempt to construct trainable weights from the other aspect of graphs (\emph{i.e.}, topological view) so that the values of $w_{ij}$ can be more flexible and capture the inherent property of graphs. Ultimately we derive the definition of Graphical Mutual Information (GMI).

\begin{definition} [Graphical Mutual Information]
	\label{De:GMI}
	The MI between the hidden vector $\Vec{h}_i$ and its support graph $\mathcal{G}_i=(\Mat{X}_i, \Mat{A}_i)$ is defined as
	\begin{eqnarray}
	\label{Eq:gmi}
	\begin{split}
	I(\Vec{h}_i; \mathcal{G}_i) \quad &\coloneqq \quad \sum\nolimits_{j}^{i_n} w_{ij} I(\Vec{h}_i; \Mat{x}_{j}) + I(w_{ij}; \Mat{a}_{ij}),\\ 
	&\quad\quad\text{with}\quad w_{ij}=\sigma(\Vec{h}_i^{\mathrm{T}}\Vec{h}_j), 
	\end{split}
	\end{eqnarray}
	where the definitions of both $\Vec{x}_j$ and $i_n$ are as same as Theorem~\ref{Th:FMI}, $\Vec{a}_{ij}$ is the edge weight/feature in the adjacency matrix $\Mat{A}$, and $\sigma(\cdot)$ is a sigmoid function.
\end{definition}

Intuitively, weight $w_{ij}$ in the first term of Eq.~\eqref{Eq:gmi} measures the contribution of a local MI to the global one. We implement the contribution of $I(\Vec{h}_i; \Mat{x}_{j})$ by the similarity between representations $\Vec{h}_i$ and $\Vec{h}_j$ (\emph{i.e.}, $w_{ij}=\sigma(\Vec{h}_i^{\mathrm{T}}\Vec{h}_j)$). Meanwhile, the term $I(w_{ij}; \Mat{a}_{ij})$ maximizes MI between $w_{ij}$ and the edge weight/feature of input graph (\emph{i.e.}, $\Vec{a}_{ij}$) to enforce $w_{ij}$ to conform to topological relations. In this sense, the degree of the contribution would be consistent with the proximity in topological structure, which is commonly accepted as a fact that $w_{ij}$ could be larger if node $j$ is ``closer'' to node $i$ and smaller otherwise. This strategy compensates for the flaw that FMI only focuses on node features and makes local MIs contribute to the global one adaptively. 
To better understand the idea of attention in this strategy, you could refer to the attention-based GCN~\cite{velivckovic2017graph}.

Note that the definition of Eq.~\eqref{Eq:gmi} is applicable for general cases. For certain specific situations, we can slightly modify Eq.~\eqref{Eq:gmi} for efficiency. For example, when dealing with unweighted graphs (namely the edge value is 1 if connected and 0 otherwise), we could replace the second MI term $I(w_{ij}; \Mat{a}_{ij})$ with a negative cross-entropy loss. Minimizing the cross-entropy also contributes to MI maximization, and it delivers a more efficient computation. We defer more details in the next section.

There are several benefits by the definition of Eq.~\eqref{Eq:gmi}. First, this kind of MI is invariant to the isomorphic transformation of input graphs. Second, it is computationally feasible as each component on the right side can be estimated by MINE. More importantly, GMI is more powerful than DGI in capturing original input information due to its explicit correlation between hidden vectors and input features of both nodes and edges in a fine-grained node-level.

\subsection{Maximization of GMI}
Now we directly maximize the right side of Eq.~\eqref{Eq:gmi} with the help of MINE. Note that MINE estimates a lower-bound of MI with the Donsker-Varadhan (DV)~\cite{donsker1983asymptotic} representation of the KL-divergence between the joint distribution and the product of the marginals.
As we focus more on maximizing MI rather than obtaining its specific value, the other non-KL alternatives such as Jensen-Shannon MI estimator (JSD)~\cite{nowozin2016f} and Noise-Contrastive estimator (infoNCE)~\cite{oord2018representation} could be employed to replace it. Based on the experimental findings and analysis in~\cite{hjelm2018learning}, we resort to JSD estimator in this paper for the sake of effectiveness and efficiency, since infoNCE estimator is sensitive to negative sampling strategies (the number of negative samples) thus may become a bottleneck for large-scale datasets with a fixed available memory. 
On the contrary, the insensitivity of JSD estimator to negative sampling strategies and its respectable performance on many tasks makes it more suitable for our task. In particular, we calculate $I(\Vec{h}_i; \Vec{x}_j)$ in the first term of Eq.~\eqref{Eq:gmi} by 
\begin{eqnarray}
\label{Eq:jsd}
I(\Vec{h}_i; \Vec{x}_j) \ =\  -sp(-\mathcal{D}_{w}(\Vec{h}_{i},\Vec{x}_{j}))-\mathbb{E}_{\tilde{\mathbb{P}}}[sp(\mathcal{D}_{w}(\Vec{h}_{i},\Vec{x'}_{j}))],
\end{eqnarray}
where $\mathcal{D}_{w}: D \times D' \rightarrow \mathbb{R}$ is a discriminator constructed by a neural network with parameter $w$. $\Vec{x'}_{j}$ is an negative sampled from $\tilde{\mathbb{P}}=\mathbb{P}$, and $sp(x)=log(1+e^{x})$ denotes the soft-plus function. 

As mentioned in \textsection~\ref{Sec:GMI}, we maximize $I(w_{ij}; \Mat{a}_{ij})$ via calculating its cross-entropy instead of using JSD estimator since the graphs we coped with in experiments are unweighted. Formally, we compute
\begin{eqnarray}
\label{Eq:cross}
\quad I(w_{ij}; \Mat{a}_{ij}) \ =\  \Mat{a}_{ij}log w_{ij} +(1-\Mat{a}_{ij})log(1-w_{ij}).
\end{eqnarray}
By maximizing $I(\Vec{h}_i; \mathcal{G}_i)$ with the sum of Eq.~\eqref{Eq:jsd} and Eq.~\eqref{Eq:cross} over all hidden vectors $\Mat{H}$, we arrive at our complete objective function for GMI optimization. Besides, we can further add trade-off parameters to balance Eq.~\eqref{Eq:jsd} and~\eqref{Eq:cross} for more flexibility.

\section{Experiments}
In this section, we empirically evaluate the performance of GMI on two common tasks: node classification (transductive and inductive) and link prediction. An additional relatively fair comparison between GMI and another two unsupervised algorithms (EP-B and DGI) further exhibits its effectiveness. Also we provide the visualization of t-SNE plots and analyze the influence of model depth.


\subsection{Datasets}
To assess the quality of our approach in each task, we adopt 4 or 5 commonly used benchmark datasets in the previous work~\cite{kipf2016semi,hamilton2017inductive,velivckovic2018deep}. Detailed statistics are given in Table~\ref{table:t1}.

In the classification task, Cora, Citeseer, and PubMed~\cite{sen2008collective}\footnote{https://github.com/tkipf/gcn} are citation networks where nodes correspond to documents and edges represent citations. Each document is associated with a bag-of-words representation vector and belongs to one of the predefined classes. Following the transductive setup in~\cite{kipf2016semi,velivckovic2018deep}, training is conducted on all nodes, and 1000 test nodes are used for evaluation. Reddit\footnote{http://snap.stanford.edu/graphsage/} is a large social network consisting of numerous interconnected Reddit posts created during September 2014~\cite{hamilton2017inductive}. Posts are treated as nodes and edges mean the same user comments. The class label is the community and our objective is to predict which community different posts belong to. PPI\footnotemark[3] is a protein-protein interaction dataset that contains multiple graphs related to different human tissues~\cite{zitnik2017predicting}. The positional gene sets, motif gene sets, and immunological signatures are viewed as node features, and each node has a totally of 121 labels given by gene ontology sets. Classifying protein functions across different PPI graphs is the goal. Following the inductive setup in~\cite{hamilton2017inductive}, on Reddit, we feed posts made in the first 20 days into the model for training, while the remaining are used for testing (with 30$\%$ used for validation); on PPI, there are 20 graphs for training, 2 for validation and 2 for testing. It should be emphasized that, for Reddit and PPI, testing is carried out on unseen (untrained) nodes and graphs, while the first three datasets are used for transductive learning.

In the link prediction task, BlogCatalog\footnote{http://dmml.asu.edu/users/xufei/datasets.html} is a social blogging website where bloggers follow each other and register their blogs under predefined 6 categories. The tags of blogs are taken as node features. Flickr\footnotemark[4] is an image sharing website where users interact with others and form
a social network. Users upload photos with 9 predefined classes and select attached tags to reflect their interests which provide attribute information. The description of Cora and PPI is omitted for brevity. Following the experimental settings and evaluation metrics in~\cite{grover2016node2vec}, given a graph with certain portions of edges removed, we aim to predict these missing links. For Cora, BlogCatalog, and Flickr, we randomly delete 20$\%$, 50$\%$, and 70$\%$ edges while ensuring that the rest of network obtained after the edge removal is connected and use the damaged network for training. About PPI, we directly treat part of the edges not seen during training as prediction targets instead of man-made edge deletion.

\subsection{Experimental Settings}\label{sec:setting}
\textbf{Encoder design.} We resort to a standard Graph Convolutional Network (GCN) model with the following layer-wise propagation rule as the encoder for both classification and link prediction tasks:
\begin{equation}\label{eq:1}
\Mat{H}^{(l+1)}=\sigma(\hat{\Mat{A}}\Mat{H}^{(l)}\Mat{W}^{(l)}),
\end{equation}
where $\hat{\Mat{A}}=\Mat{D}^{-\frac{1}{2}}\bar{\Mat{A}}\Mat{D}^{-\frac{1}{2}}$, $\bar{\Mat{A}}=\Mat{A}+\Mat{I}_{n}$, $\Mat{D}_{ii}=\sum_{j}\bar{\Mat{A}}_{ij}$, $\Mat{H}^{(l)}$ and $\Mat{H}^{(l+1)}$ are the input and output matrices of the $l$-th layer, $\Mat{W}^{(l)}$ is a layer-specific trainable weight matrix. Here the nonlinear transformation $\sigma$ we applied is the PReLU function (parametric ReLU)~\cite{he2015delving}. It should be recognized that for node $i$, the neighborhood $\Mat{X}_{i}$ in its support graph $\mathcal{G}_{i}$ contains node $i$ itself as self-loops are inserted through $\bar{\Mat{A}}$. 

To be more specific, the encoder we employed on Citeseer and PubMed is a one-layer GCN with the output dimension as $D'=512$. And on Cora, Reddit, BlogCatalog, Flickr, and PPI, we utilize a two-layer GCN as our encoder. Here, we have hidden dimensions as $D'=512$ in each GCN layer. Note that utilizing the similar GCN encoder for both transductive and inductive classification task makes our proposed method easier to follow and scale to large networks than DGI, since DGI has to design varying encoders to adapt to distinct learning tasks, especially the encoders used in inductive tasks are too intricate and complicated, which are not friendly to practical applications. 

\noindent\textbf{Discriminator design.} The discriminator in Eq.~\eqref{Eq:jsd} scores the input-output feature pairs through a simple bilinear function, which is similar to the discriminator used in~\cite{oord2018representation}:
\begin{equation}\label{eq:4}
\mathcal{D}(\Vec{h}_i, \Vec{x}_j)=\sigma(\Vec{h}_{i}^{T}\Mat{\Theta}\Vec{x}_{j}),
\end{equation}
where $\Mat{\Theta}$ represents a trainable scoring matrix and the activation function $\sigma$ we employed is the sigmoid aiming at converting scores into probabilities of $(\Vec{h}_{i}, \Vec{x}_{j})$ being a positive example.

\noindent\textbf{Implementation details.} \label{Sec:imp}
Actually, for the weight $w_{ij}$ of the first term in Eq.~\eqref{Eq:gmi}, we have two ways to get its value in experiments. The first is to keep $w_{ij}=\sigma(\Vec{h}_i^{\mathrm{T}}\Vec{h}_j)$, which makes local MIs contribute to the global one adaptively, and we term this variant GMI-adaptive. The other is to let $w_{ij}=\frac{1}{i_n}$, $i.e.$, the left endpoint of the interval where $w_{ij}$ belongs (refer to Theorem~\ref{Th:FMI}), which means the contribution of each local MI is equal, and we term this variant GMI-mean. Here both GMI-mean and GMI-adaptive are included in the scope of comparison with baselines. 

All experiments are implemented in PyTorch~\cite{ketkar2017introduction} with Glorot initialization~\cite{glorot2010understanding} and conducted on a single Tesla P40 GPU. In preprocessing, we perform row normalization on Cora, Citeseer, PubMed, BlogCatalog, and Flickr following~\cite{kipf2016semi}, and apply the processing strategy in~\cite{hamilton2017inductive} on Reddit and PPI. During training, we use Adam optimizer~\cite{kingma2014adam} with an initial learning rate of 0.001 on all seven datasets. Suggested by~\cite{velivckovic2018deep}, we adopt an early stopping strategy with a window size of 20 on Cora, Citeseer, and PubMed, while training the model for a fixed number of epochs on the inductive datasets (20 on Reddit, 50 on PPI). The number of negative samples is set to 5. Due to the large scale of Reddit and PPI, we need to use the subsampling skill introduced in~\cite{hamilton2017inductive} to make them fit into GPU memory. In detail, a minibatch of 256 nodes is first selected, and then for each selected node, we uniformly sample 8 and 5 neighbors at its first and second-level neighborhoods, respectively. We adopt the one-hop neighborhood to construct the support graph in experiments and utilize $\Mat{X}\Mat{W}^{(0)}$ (\emph{i.e.}, a compressed input feature) to calculate FMI since using the original input feature $\Mat{X}$ causes GPU memory overflow. The trade-off parameters are tuned in the range of [0,1] to balance Eq.~\eqref{Eq:jsd} and Eq.~\eqref{Eq:cross}. The Batch Normalization strategy ~\cite{ioffe2015batch} is employed to train our model on Reddit and PPI. 

\noindent\textbf{Evaluation metrics.} For the classification task, we provide the learned embeddings across the training set to the logistic regression classifier and give the results on the test nodes~\cite{kipf2016semi,hamilton2017inductive}. Specifically, in transductive learning, we adopt the mean classification accuracy after 50 runs to evaluate the performance, while the micro-averaged F1 score averaged after 50 runs is used in inductive learning. And for PPI, suggested by~\cite{velivckovic2018deep}, we standardize the learned embeddings before feeding them into the logistic regression classifier. For the link prediction task, the criteria we adopted is AUC which is the area under the ROC curve. The negative samples involved in the calculation of AUC are generated by randomly selecting an equal number of node pairs with no connections in the original graph. The closer the AUC score approaches 1, the better the performance of the algorithm is. Similarly, we report the AUC score averaged after 10 runs.

\begin{table*}[t]
	\centering
	\tabcolsep=11 pt  
	\caption{Classification accuracies (with standard deviation) in percent on transductive tasks and micro-averaged F1 scores on inductive tasks. The third column illustrates the data used by each algorithm in the training phase, where $\mathbf{X}$, $\mathbf{A}$, and $\mathbf{Y}$ denotes features, adjacency matrix, and labels, respectively.}\label{table:t2} 
	\begin{tabular}{clcccc}
		\toprule
		& \multirow{2}{*}{\textbf{Algorithm}} &\textbf{Training data} & \multicolumn{3}{c}{\textbf{Transductive tasks}}\\
		&&$\mathbf{X\qquad A\qquad Y}$&\textbf{Cora}&\textbf{Citeseer}&\textbf{PubMed}\\
		\midrule
		\multirow{5}{*}{\textbf{Unsupervised}}& Raw features &\checkmark \qquad \ \; \qquad \ \;  & 56.6 $\pm$ 0.4 & 57.8 $\pm$ 0.2 & 69.1 $\pm$ 0.2\\
		&DeepWalk & \qquad \ \; \checkmark \qquad \ \; & 67.2 & 43.2 & 65.3\\
		&EP-B & \checkmark \qquad \checkmark \qquad \ \;& 78.1 $\pm$ 1.5 &  71.0 $\pm$ 1.4 & 79.6 $\pm$ 2.1\\
		&DGI & \checkmark \qquad \checkmark \qquad \ \; & 82.3 $\pm$ 0.6 & 71.8 $\pm$ 0.7 & 76.8 $\pm$ 0.6\\
		&\textbf{GMI-mean} (ours) & \checkmark \qquad \checkmark \qquad \ \; & \textbf{82.7} $\pm$ \textbf{0.2} & \textbf{73.0} $\pm$ \textbf{0.3} & \textbf{80.1} $\pm$ \textbf{0.2}\\
		&\textbf{GMI-adaptive} (ours) & \checkmark \qquad \checkmark \qquad \ \; & \textbf{83.0} $\pm$ \textbf{0.3}& \textbf{72.4} $\pm$ \textbf{0.1}& \textbf{79.9} $\pm$ \textbf{0.2} \\
		\midrule
		\multirow{4}{*}{\textbf{Supervised}}&LP & \qquad \ \; \checkmark \qquad \checkmark & 68.0 & 45.3 & 63.0\\
		&Planetoid-T & \checkmark \qquad \checkmark \qquad \checkmark & 75.7 & 62.9 & 75.7\\
		&GCN & \checkmark \qquad \checkmark \qquad \checkmark & 81.5 & 70.3 & 79.0\\
		&GAT & \checkmark \qquad \checkmark \qquad \checkmark & 83.0 $\pm$ 0.7 & 72.5 $\pm$ 0.7 & 79.0 $\pm$ 0.3\\
		&GWNN & \checkmark \qquad \checkmark \qquad \checkmark & 82.8  & 71.7 & 79.1\\
		\midrule
		&\multirow{2}{*}{\textbf{Algorithm}} &\textbf{Training data}& \multicolumn{3}{c}{\textbf{Inductive tasks}}\\
		&&$\mathbf{X\qquad A\qquad Y}$& \multicolumn{2}{c}{\textbf{Reddit}}&\textbf{PPI}\\
		\midrule
		\multirow{9}{*}{\textbf{Unsupervised}}&Raw features & \checkmark \qquad \ \; \qquad \ \; & \multicolumn{2}{c}{58.5} & 42.2\\
		&DeepWalk & \qquad \ \; \checkmark \qquad \ \; & \multicolumn{2}{c}{32.4} & -\\
		&DeepWalk+features & \checkmark \qquad \checkmark \qquad \ \; & \multicolumn{2}{c}{69.1} & -\\
		&GraphSAGE-GCN & \checkmark \qquad \checkmark \qquad \ \; & \multicolumn{2}{c}{90.8} & 46.5 \\
		&GraphSAGE-mean & \checkmark \qquad \checkmark \qquad \ \; & \multicolumn{2}{c}{89.7} & 48.6\\
		&GraphSAGE-LSTM & \checkmark \qquad \checkmark \qquad \ \; & \multicolumn{2}{c}{90.7} & 48.2\\
		&GraphSAGE-pool & \checkmark \qquad \checkmark \qquad \ \; & \multicolumn{2}{c}{89.2} & 50.2\\
		&DGI & \checkmark \qquad \checkmark \qquad \ \; & \multicolumn{2}{c}{94.0 $\pm$ 0.10} & 63.8 $\pm$ 0.20\\
		&\textbf{GMI-mean} (ours) & \checkmark \qquad \checkmark \qquad \ \; & \multicolumn{2}{c}{\textbf{95.0} $\pm$ \textbf{0.02}} & \textbf{65.0} $\pm$ \textbf{0.02}\\
		&\textbf{GMI-adaptive} (ours) & \checkmark \qquad \checkmark \qquad \ \; & \multicolumn{2}{c}{\textbf{94.9} $\pm$ \textbf{0.02}}& \textbf{64.6} $\pm$ \textbf{0.03}\\
		\midrule
		\multirow{3}{*}{\textbf{Supervised}}&GAT  & \checkmark \qquad \checkmark \qquad \checkmark & \multicolumn{2}{c}{-} & 97.3 $\pm$ 0.20 \\
		&FastGCN & \checkmark \qquad \checkmark \qquad \checkmark & \multicolumn{2}{c}{93.7} & -\\
		&GaAN  & \checkmark \qquad \checkmark \qquad \checkmark & \multicolumn{2}{c}{96.4 $\pm$ 0.03} & 98.7 $\pm$ 0.02\\
		\bottomrule
	\end{tabular}
\end{table*}

\subsection{Classification}
\textbf{Transductive learning.} Table~\ref{table:t2} reports the mean classification accuracy of our method and other baselines on transductive tasks. Here, results for EP-B~\cite{duran2017learning}, DGI~\cite{velivckovic2018deep}, Planetoid-T~\cite{yang2016revisiting}, GAT~\cite{velivckovic2017graph}, as well as GWNN~\cite{xu2018graph} are taken from their original papers, and results for DeepWalk~\cite{perozzi2014deepwalk}, LP (Label Propagation)~\cite{zhu2003semi}, and GCN~\cite{kipf2016semi} are copied from Kipf \& Welling~\cite{kipf2016semi}. As for raw features, we feed them into a logistic regression classifier for training and give the results on the test features\footnote{Strictly speaking, this experiment belongs to the inductive learning as testing is conducted on unseen features. But for comparison, we put it in this part.}. Although we provide experimental results of both supervised and unsupervised methods, in this paper, we focus more on comparing against unsupervised ones which are consistent with our setup.

As can be observed, our proposed GMI-mean and GMI-adaptive, compared with other unsupervised methods, achieve the best classification accuracy across all three datasets. We consider this strong performance benefits from the idea of attempting to directly maximize graphical MI between input and output pairs of the encoder $\mathcal{E}$ at a fine-grained node-level. Therefore, the encoded representation maximally preserves the information of node features and topology in $\mathcal{G}$, which contributes to classification. By contrast, EP-B ignores the underlying information between input data and learned representations, and DGI stays in a graph/patch-level MI maximization, which restricts their capability of preserving and extracting the original input information into embedding space. Thus slightly weak performance on classification tasks. Besides, without the guidance of labels, our method exhibits comparable results to some supervised models like GCN and GAT, even better results than them on Citeseer and PubMed. We believe that representations learned via GMI maximization between inputs and outputs inherit the rich information in graph $\mathcal{G}$ which is enough for classification. More notable is that many available labels are given based on the information in $\mathcal{G}$ as well. So keeping as much information as possible from the input can compensate for the information provided by the label to some extent, which sustains the performance of GMI in downstream graph mining tasks. It could be claimed that learning from original inputs without labels promises the potential for higher quality representations than the supervised pattern as the extreme sparsity of the training labels may suffer from the threat of overfitting or the correctness of given labels might not be reliable.

\begin{table*}[t]
	\centering
	\tabcolsep=17 pt 
	\caption{Comparison with different objective functions in terms of classification accuracies (with standard deviation) in percent on Cora, Citeseer, and PubMed, and micro-averaged F1 scores in percent on Reddit and PPI.}\label{table:t4}  
	\begin{tabular}{cccccc}
		\toprule
		\multirow{2}{*}{\textbf{Algorithm}} & \multicolumn{3}{c}{\textbf{Transductive}} & \multicolumn{2}{c}{\textbf{Inductive}}\\
		&\textbf{Cora} & \textbf{Citeseer} & \textbf{PubMed} & \textbf{Reddit} & \textbf{PPI}\\
		\midrule
		EP-B loss & 79.4 $\pm$ 0.1 & 69.3 $\pm$ 0.2 & 78.6 $\pm$ 0.2 & 93.8 $\pm$ 0.03 & 61.8 $\pm$ 0.04\\
		DGI loss & 82.2 $\pm$ 0.2 & 72.2 $\pm$ 0.2 & 78.9 $\pm$ 0.3 & 94.3 $\pm$ 0.02 & 62.3 $\pm$ 0.02\\
		\textbf{FMI} (ours) & 78.3 $\pm$ 0.1 & 72.0 $\pm$ 0.2 & \textbf{79.1} $\pm$ \textbf{0.3} & \textbf{94.7} $\pm$ \textbf{0.03} & \textbf{64.8} $\pm$ \textbf{0.03}\\
		\textbf{GMI-mean} (ours) & \textbf{82.7} $\pm$ \textbf{0.1} & \textbf{73.0} $\pm$ \textbf{0.3} & \textbf{80.1} $\pm$ \textbf{0.2} & \textbf{95.0} $\pm$ \textbf{0.02} & \textbf{65.0} $\pm$ \textbf{0.02}\\
		\textbf{GMI-adaptive} (ours) & \textbf{83.0} $\pm$ \textbf{0.3}& \textbf{72.4} $\pm$ \textbf{0.1}& \textbf{79.9} $\pm$ \textbf{0.2}& \textbf{94.9} $\pm$ \textbf{0.02}& \textbf{64.6} $\pm$ \textbf{0.03}\\
		\bottomrule
	\end{tabular}
\end{table*}

\begin{table*}[t]
	\centering
	\caption{AUC scores (with standard deviation) in percent for link prediction. The percentage is the ratio of edge removal.}
	\begin{tabular}{ccccccccccc}
		\toprule
		\multirow{2}{*}{\textbf{Algorithm}}&\multicolumn{3}{c}{\textbf{Cora}}&\multicolumn{3}{c}{\textbf{BlogCatalog}}&\multicolumn{3}{c}{\textbf{Flickr}}&\textbf{PPI}\\
		& \textbf{20.0$\%$} & \textbf{50.0$\%$} & \textbf{70.0$\%$} & \textbf{20.0$\%$} & \textbf{50.0$\%$} & \textbf{70.0$\%$} & \textbf{20.0$\%$} & \textbf{50.0$\%$} & \textbf{70.0$\%$} & \textbf{22.7$\%$} \\
		\midrule
		DGI & 95.6$\pm$0.3 & 94.6$\pm$0.4 & 94.4$\pm$0.2 & 77.2$\pm$0.4& 76.4$\pm$0.4& 75.5$\pm$0.3 & 90.3$\pm$0.3 & 89.0$\pm$0.4 & 74.1$\pm$0.7 &77.4$\pm$0.1\\
		\textbf{FMI} (ours)  & \textbf{97.2}$\pm$\textbf{0.2} & \textbf{95.2}$\pm$\textbf{0.1} & \textbf{95.0}$\pm$\textbf{0.1} & \textbf{81.2}$\pm$\textbf{0.2} & \textbf{79.5}$\pm$\textbf{0.4} & 75.1$\pm$0.2 & \textbf{92.7}$\pm$\textbf{0.3} & \textbf{92.2}$\pm$\textbf{0.3} & \textbf{90.6}$\pm$\textbf{0.4} &\textbf{79.8}$\pm$0.2\\
		\textbf{GMI} (ours) & \textbf{97.9}$\pm$\textbf{0.3} & \textbf{96.4}$\pm$\textbf{0.2} & \textbf{96.3}$\pm$\textbf{0.1} & \textbf{84.1}$\pm$\textbf{0.3} & \textbf{83.6}$\pm$\textbf{0.2} & \textbf{82.5}$\pm$\textbf{0.1} & \textbf{92.0}$\pm$\textbf{0.2} & \textbf{90.1}$\pm$\textbf{0.3} & \textbf{88.5}$\pm$\textbf{0.2} & \textbf{80.0}$\pm$\textbf{0.2}\\
		\bottomrule
	\end{tabular}\label{table:t5}
\end{table*}

\textbf{Inductive learning.} Table~\ref{table:t2} also summarizes the micro-averaged F1 scores of GMI and other baselines on Reddit and PPI. We cite the results of DGI, GAT, FastGCN~\cite{chen2018fastgcn}, and GaAN~\cite{zhang2018gaan} in their original papers, while results for the rest seven compared methods are extracted from Hamilton et al.~\cite{hamilton2017inductive} (here we reuse the unsupervised GraphSAGE results to match our setup). Similarly, the comparison with unsupervised algorithms is the emphasis of our work.

GMI-mean and GMI-adaptive successfully outperform all other competing unsupervised algorithms on Reddit and PPI, which substantiates the effectiveness of GMI maximization in the inductive classification domain (generalization to unseen nodes). Interestingly, the result of our method on Reddit is competitive with some advanced supervised models, but the situation on PPI is quite different. After conducting further analysis, we note that 42$\%$ of nodes have zero feature values in PPI, which means the feature matrix $\mathbf{X}$ is very sparse~\cite{hamilton2017inductive}. In this case, directly and merely relying on input graph $\mathcal{G}=(\Mat{X},\Mat{A})$ limits the performance of unsupervised approaches including DGI and our method, whereas learning in a supervised fashion exhibits much better performance due to the auxiliary information brought by additional labels.

\textbf{Evaluation on two variants of GMI.} According to Table~\ref{table:t2}, the two variants of GMI (GMI-mean and GMI-adaptive), which use different strategies to measure the contribution of each local MI (details in \textsection~\ref{sec:setting}), achieve competitive results with each other, but GMI-adaptive exhibits slightly weaker performance than GMI-mean. Through further analysis, we assume that it might be due to the difficulties in training brought by the nature of adaptive learning. Maybe the performance of GMI-adaptive could be improved with the help of an advanced training strategy. In this sense, GMI-mean is more practical and feasible, thus it can be regarded as a representative in practice.

\subsection{Effectiveness of Objective Function}
To further clarify the effectiveness of maximizing graphical MI in unsupervised graph representation learning and provide a relatively fair comparison with DGI and EP-B (two unsupervised algorithms), we replace our objective function with their loss functions, respectively, while keeping other experimental settings unchanged. Table~\ref{table:t4} lists the results under the transductive and inductive setup. As can be observed, GMI (GMI-mean and GMI-adaptive) achieves stronger performance across all five datasets, which reflects DGI and EP-B lack some consideration in graph representation learning task. Specifically, EP-B loss only imposes constraints on each node and its neighbors at the output level (embedding space), it ignores the interaction between input and output pairs of the encoder, which results in its poor ability to retain the valid information in $\mathcal{G}$. For DGI, although it correlates hidden representations with their original input features implicitly, it discusses MI at the graph/patch-level which is somewhat coarse. Interestingly, compared with DGI, our FMI (without topology information) gains improvements more significantly with the increase of graph size. We attribute this discovery to the fact that the performance degradation of the readout function makes DGI lose certain useful information for node classification with the increase in graph size, although it exhibits good performance on small graphs such as Cora and Citeseer. When the topology of input graph is reflected, our GMI outperforms all other kinds of losses on all datasets. Furthermore, note that the whole training process of GMI is similar to the training of discriminatorys in generative models~\cite{goodfellow2014generative,nowozin2016f}, and GMI empirically exhibits a comparable training speed with EP-B and DGI on the largest dataset Reddit, which demonstrates its good scalability.

\begin{table*}[t]
	\centering
	\caption{Visualization of t-SNE embeddings from raw features, FMI, a learned GMI, and DGI on Cora and Citeseer.}
	\begin{tabular}{ccccc}
		\toprule
		&\textbf{Raw features}&\textbf{FMI (ours)}&\textbf{GMI (ours)}&\textbf{DGI}\\
		\midrule
		\multirow{8}{*}{\textbf{Cora}}&\multirow{8}{*}{\begin{minipage}[t]{0.2\linewidth}
				\includegraphics[width=1.6in,height=1.25in]{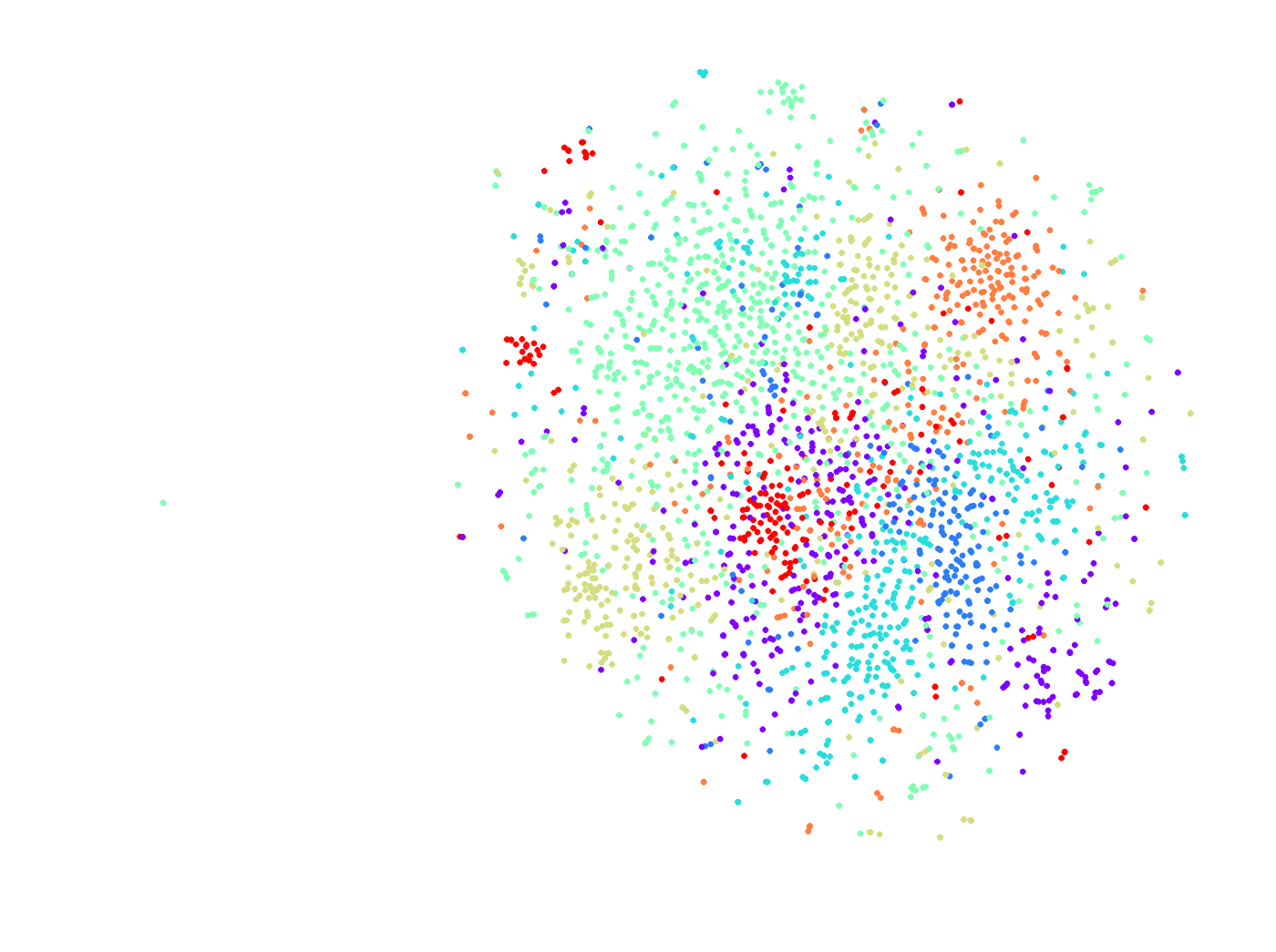}
		\end{minipage}}&\multirow{8}{*}{\begin{minipage}[t]{0.2\linewidth}
				\includegraphics[width=1.6in,height=1.25in]{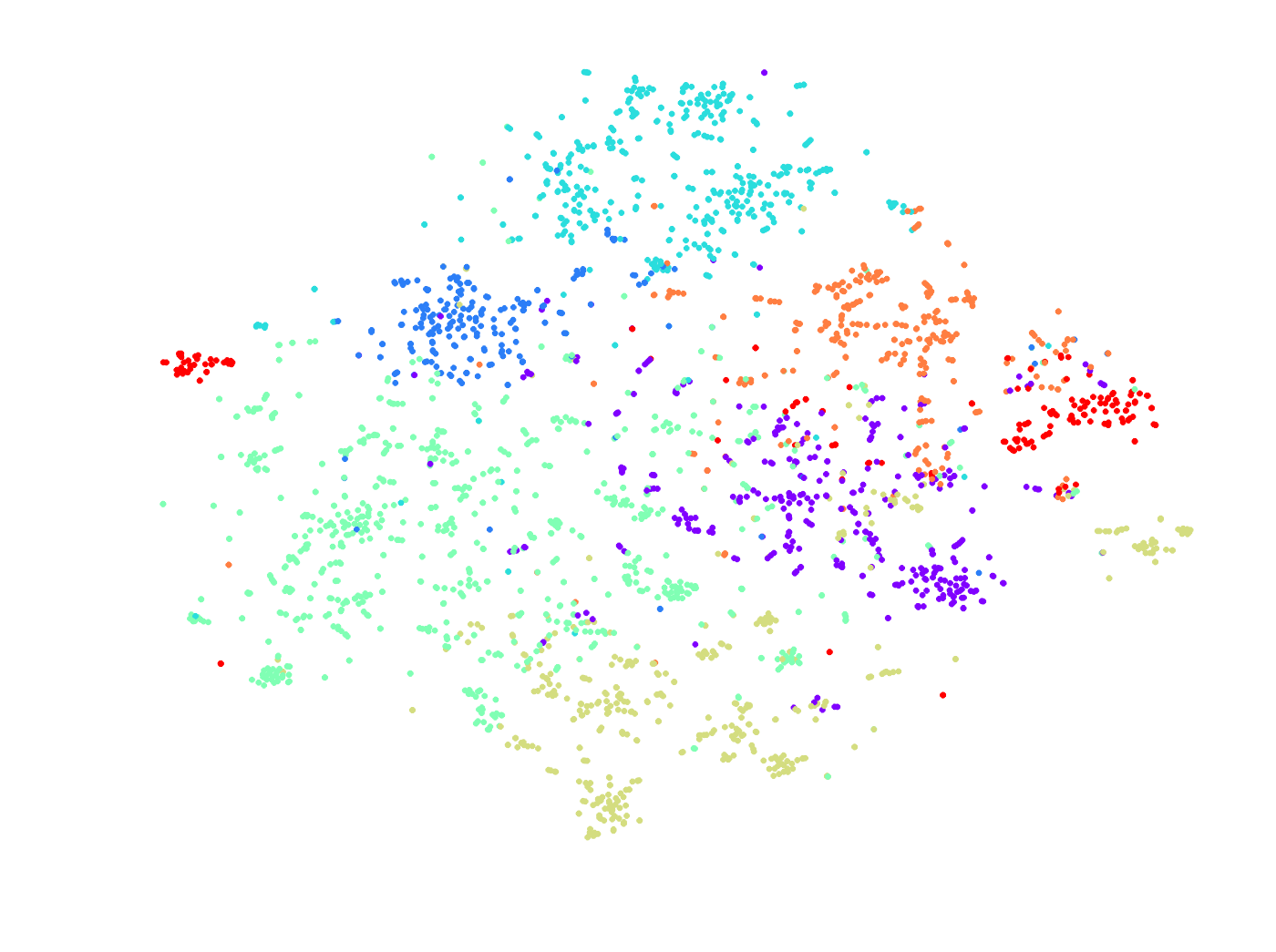}
		\end{minipage}}&\multirow{8}{*}{\begin{minipage}[t]{0.2\linewidth}
				\includegraphics[width=1.6in,height=1.25in]{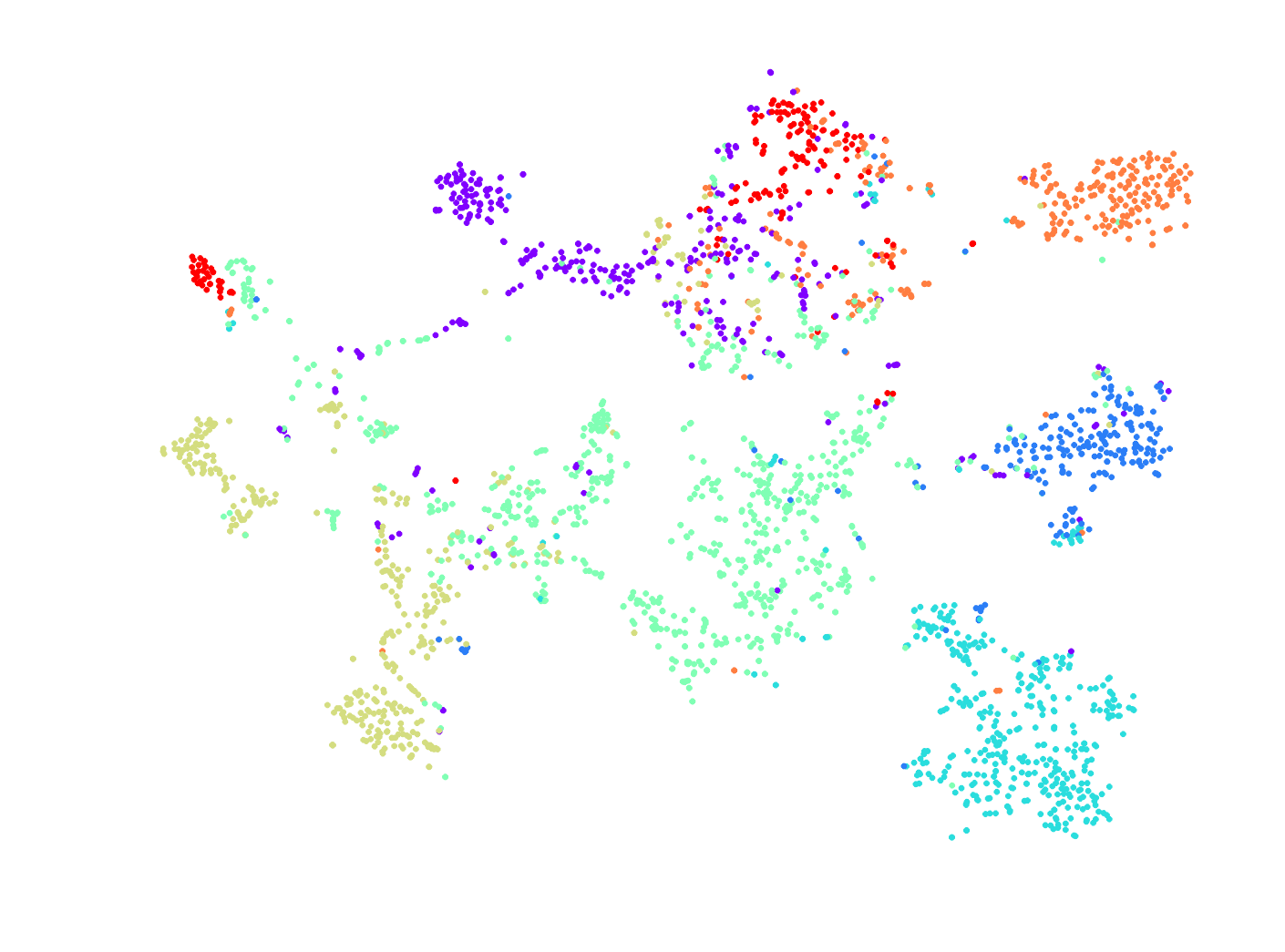}
		\end{minipage}}&\multirow{8}{*}{\begin{minipage}[t]{0.2\linewidth}
				\includegraphics[width=1.6in,height=1.25in]{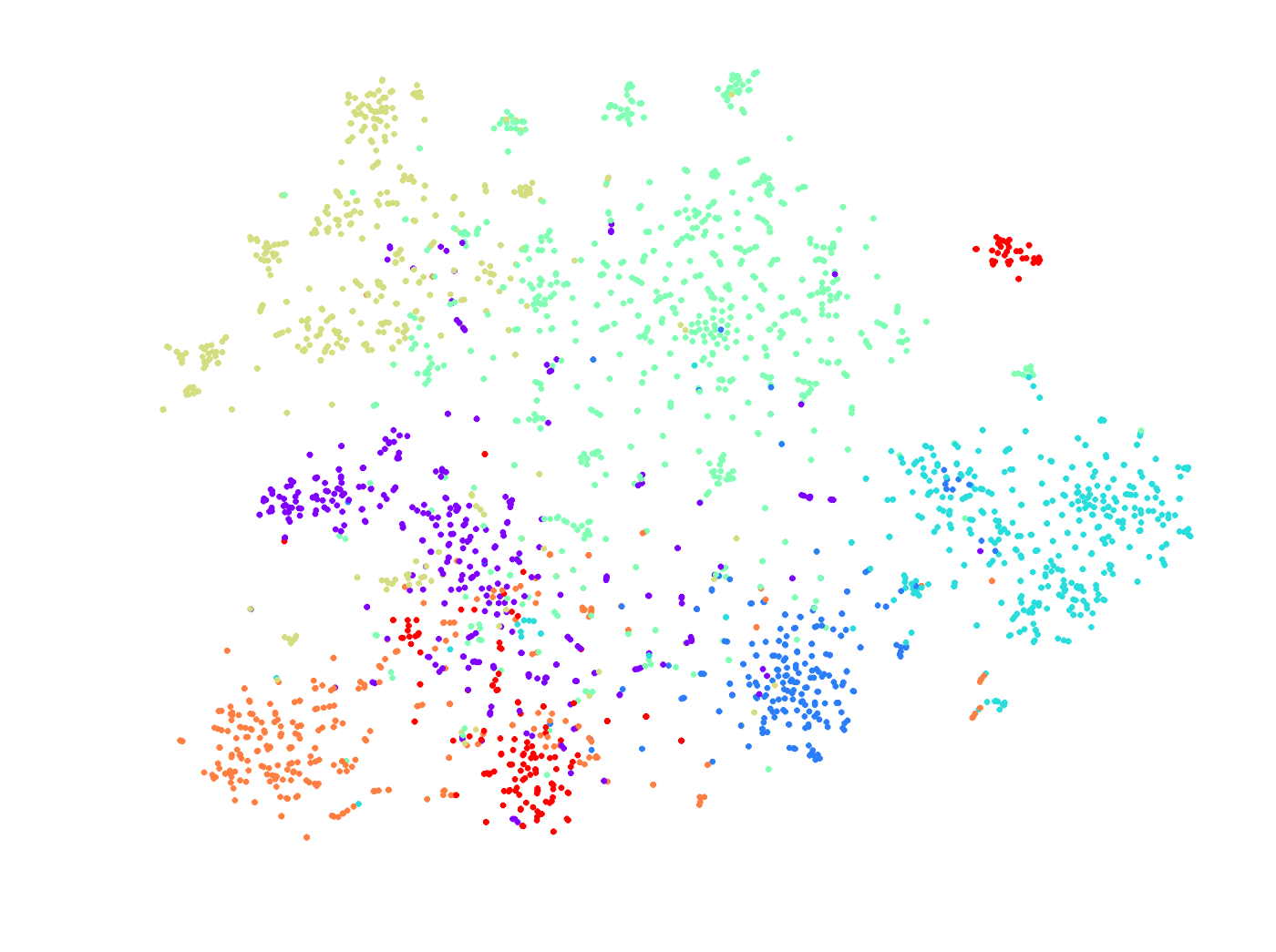}
		\end{minipage}}\\
		&&&&\\
		&&&&\\
		&&&&\\
		&&&&\\
		&&&&\\
		&&&&\\
		&&&&\\
		\midrule
		\multirow{8}{*}{\textbf{Citeseer}}&\multirow{8}{*}{\begin{minipage}[t]{0.2\linewidth}
				\includegraphics[width=1.6in,height=1.25in]{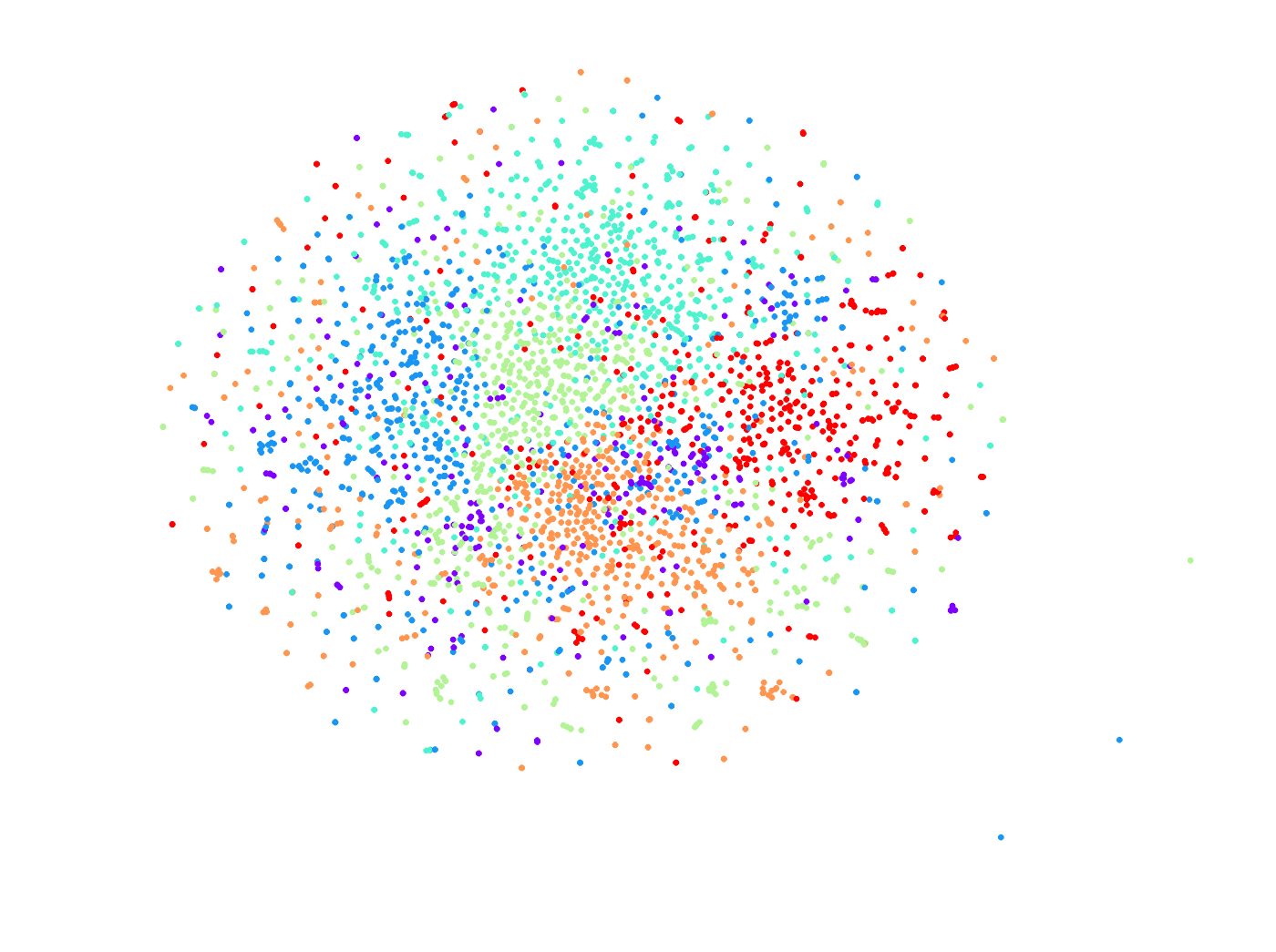}
		\end{minipage}}&\multirow{8}{*}{\begin{minipage}[t]{0.2\linewidth}
				\includegraphics[width=1.6in,height=1.25in]{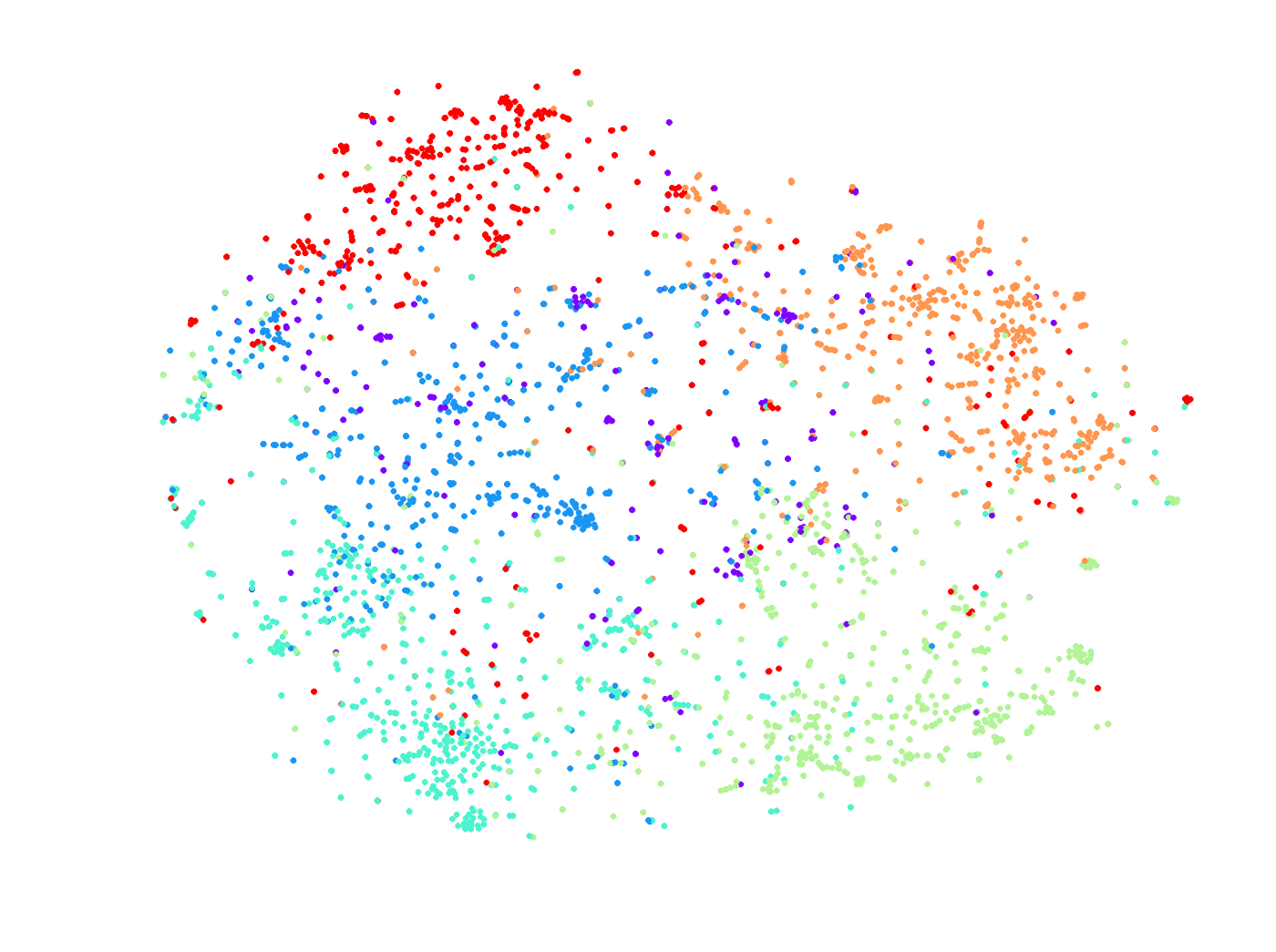}
		\end{minipage}}&\multirow{8}{*}{\begin{minipage}[t]{0.2\linewidth}
				\includegraphics[width=1.6in,height=1.25in]{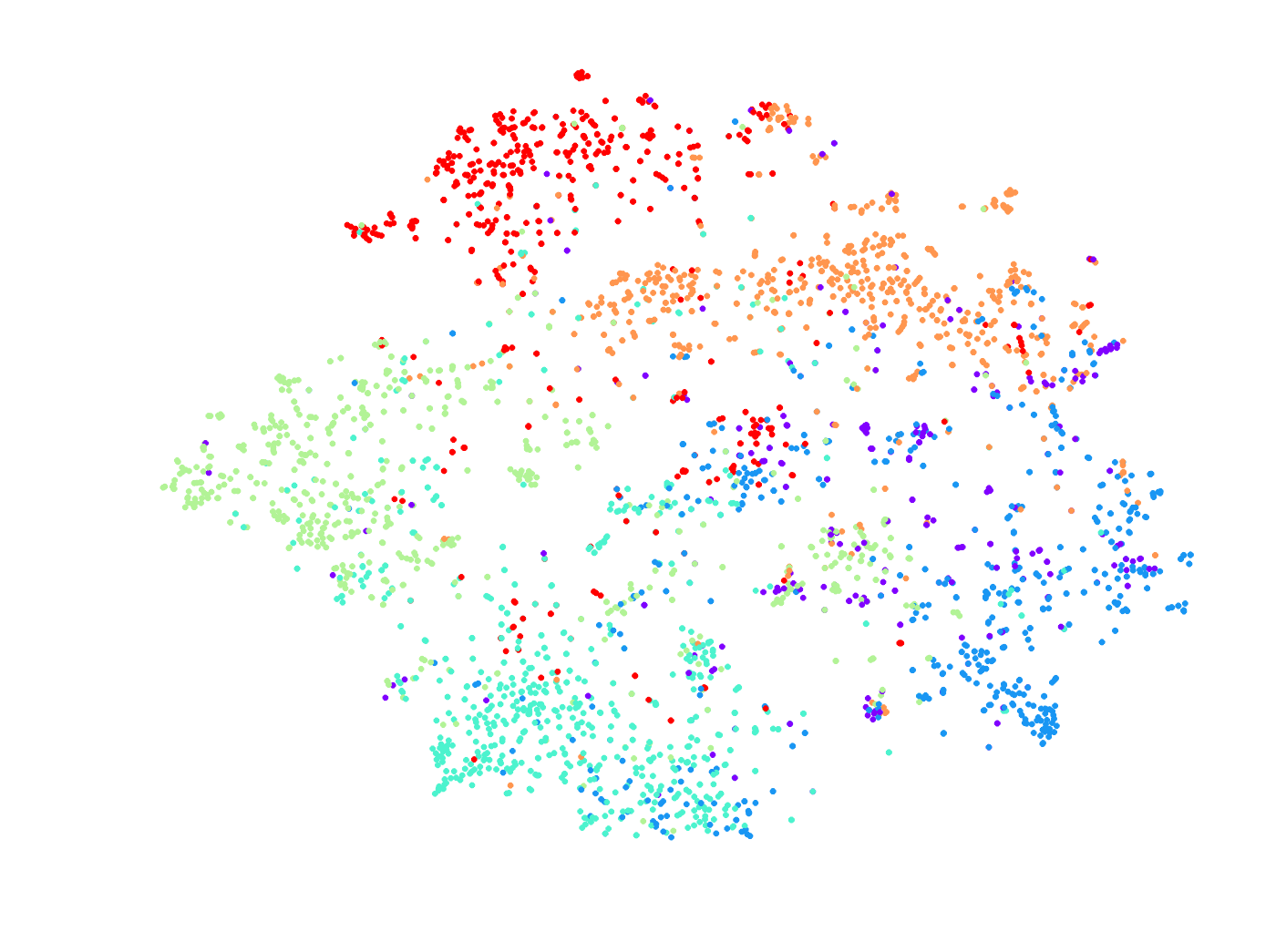}
		\end{minipage}}&\multirow{8}{*}{\begin{minipage}[t]{0.2\linewidth}
				\includegraphics[width=1.6in,height=1.25in]{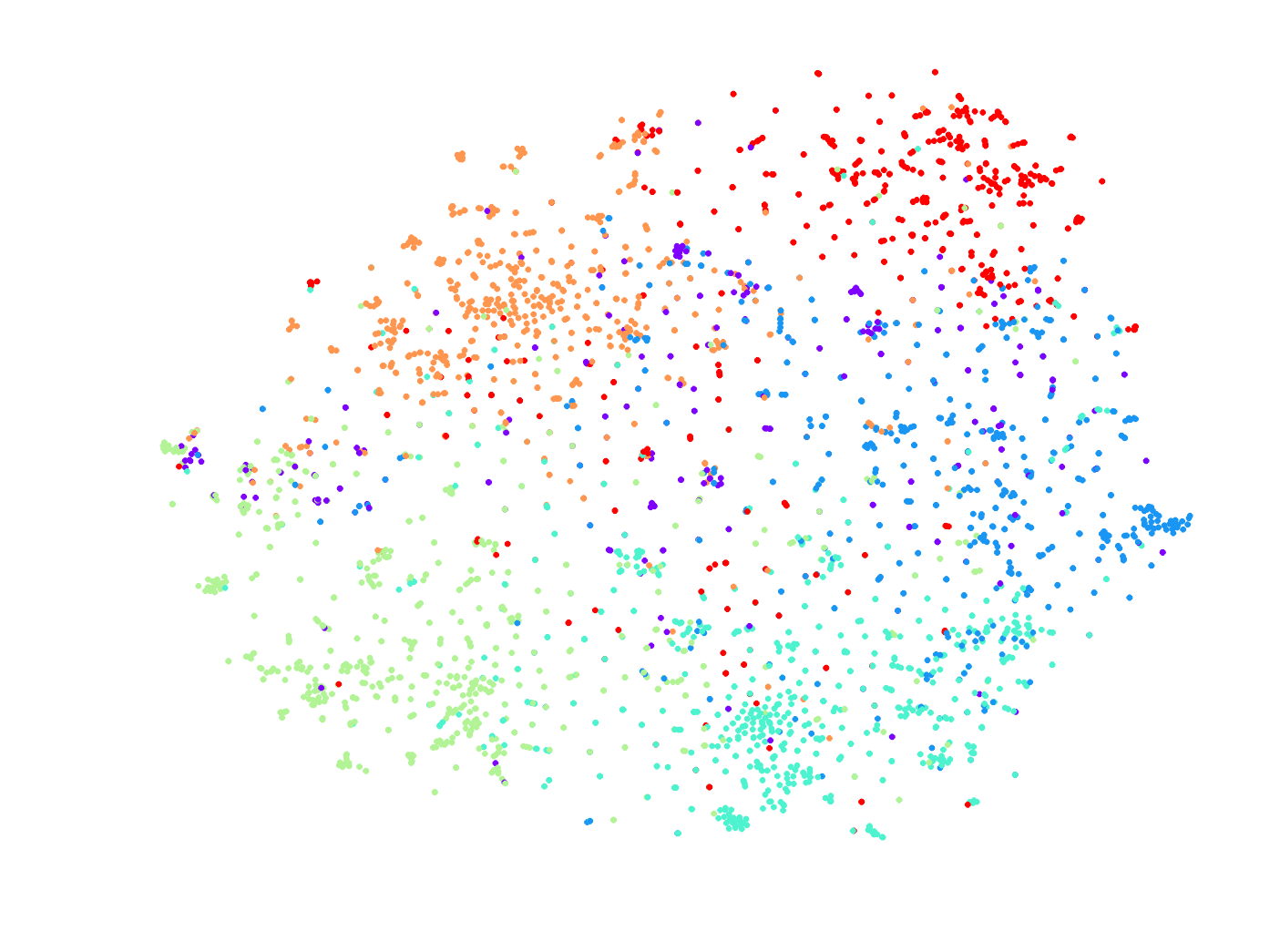}
		\end{minipage}}\\
		&&&&\\
		&&&&\\
		&&&&\\
		&&&&\\
		&&&&\\
		&&&&\\
		&&&&\\
		\bottomrule
	\end{tabular}\label{table:t6}
\end{table*}

\subsection{Link Prediction}
Based on the above experimental results, we find that DGI is a strong competitor to GMI in the scope of unsupervised algorithms. Therefore, in this section, we intend to further investigate the performance of DGI and GMI in another mining task---link prediction. Here we choose FMI and GMI-mean to compare with DGI. Table~\ref{table:t5} reports their AUC scores on four different datasets. Under different edge removal rates, GMI and FMI both remarkably outperform DGI (except FMI in 70.0$\%$ BlogCatalog), showing that measuring graphical MI between input graph and output representations in a fine-grained pattern is capable of capturing rich information in inputs and delivering good generalization ability. About DGI, for one thing, its graph/patch-level MI maximization which is relatively coarse limits its performance in such a fine link prediction task; for another, the inappropriateness of corruption function weakens the ability of DGI to learn accurate representations to predict missing links. Recall that the negative sample for the discriminator in DGI is generated by corrupting the original input graph, and a well-designed corruption function is indispensable which needs some skillful strategies~\cite{velivckovic2018deep}. In this task, we still adopt feature shuffling function which shows the best results in the classification task to build negative samples. But in the case where an input graph is incomplete in terms of topological links, the guidance provided by this corrupted graph as a negative label in the discriminator becomes unreliable due to the inaccuracy of input graph, leading to poor performance. Therefore, the necessity of task-oriented corruption function is a weakness of DGI. In contrast, our GMI is free from this issue by eliminating the corruption function and directly maximizing graphical MI between inputs and outputs of the encoder. Furthermore, it can be observed that FMI is competitive to GMI in most cases, even on Flickr FMI is superior to GMI. We assume it to the benefits brought by the direct and elaborate feature mutual information maximization at a node-level. Based on the Homophily hypothesis~\cite{mcpherson2001birds} (\emph{i.e.}, entities in the network with similar features are likely to interconnect), the input feature information preserved in learned embeddings makes FMI owns the good capability of inferring missing links. 
\begin{figure}[t] 
	\centering 
	\includegraphics[width=0.48\textwidth]{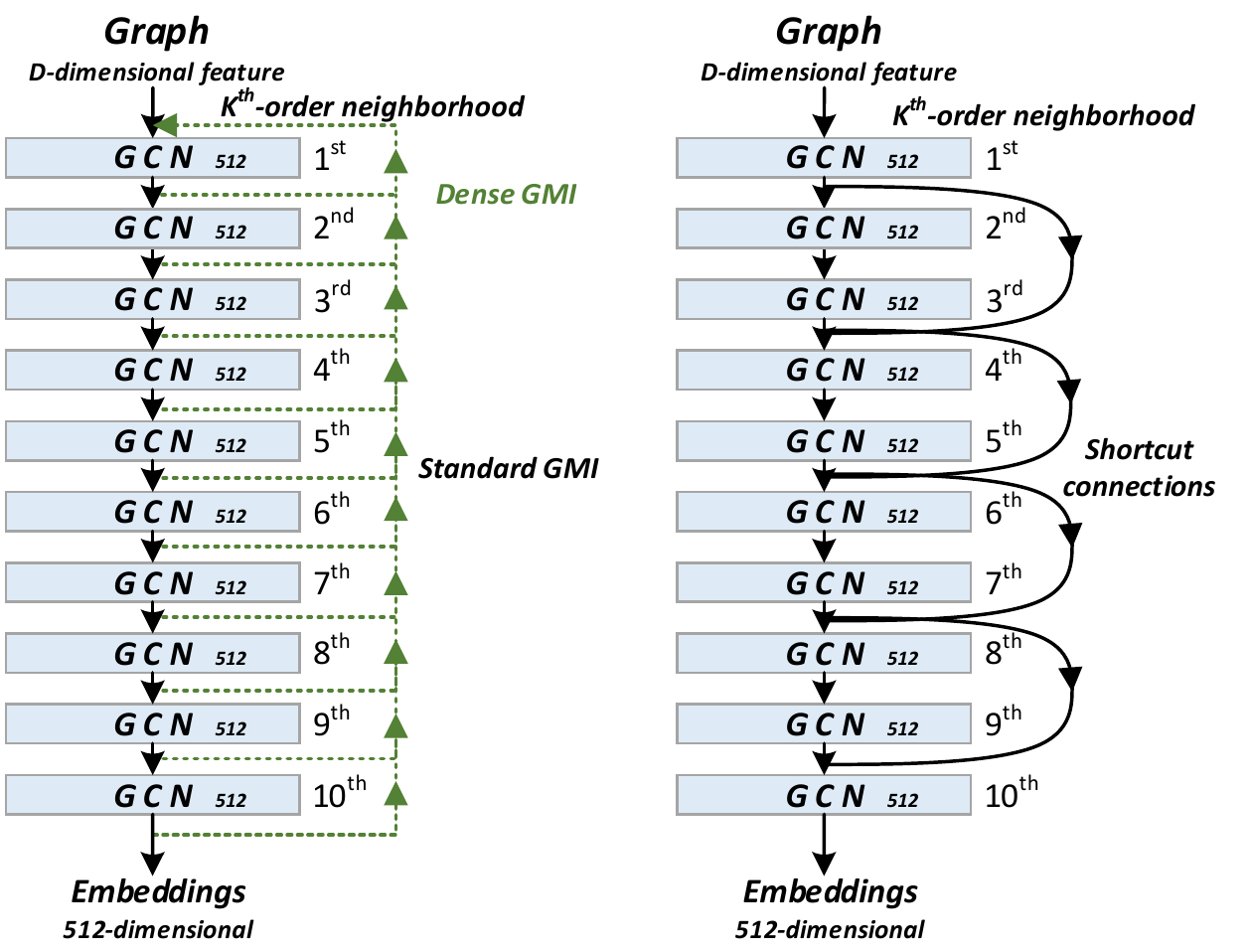}
	\caption{Network architecture for standard GMI model (\textbf{left}) and a residual variant (\textbf{right}). Besides, for the standard GMI model, we have a variant that achieves GMI maximization between the output of each layer and input graph, called as dense GMI.} 
	\label{fig:graph3}
	\vskip -0.2in
\end{figure}

\begin{figure*}[t]
	\begin{minipage}[t]{0.4\linewidth}
		\centering
		\includegraphics[width=2.5in,height=1.78in]{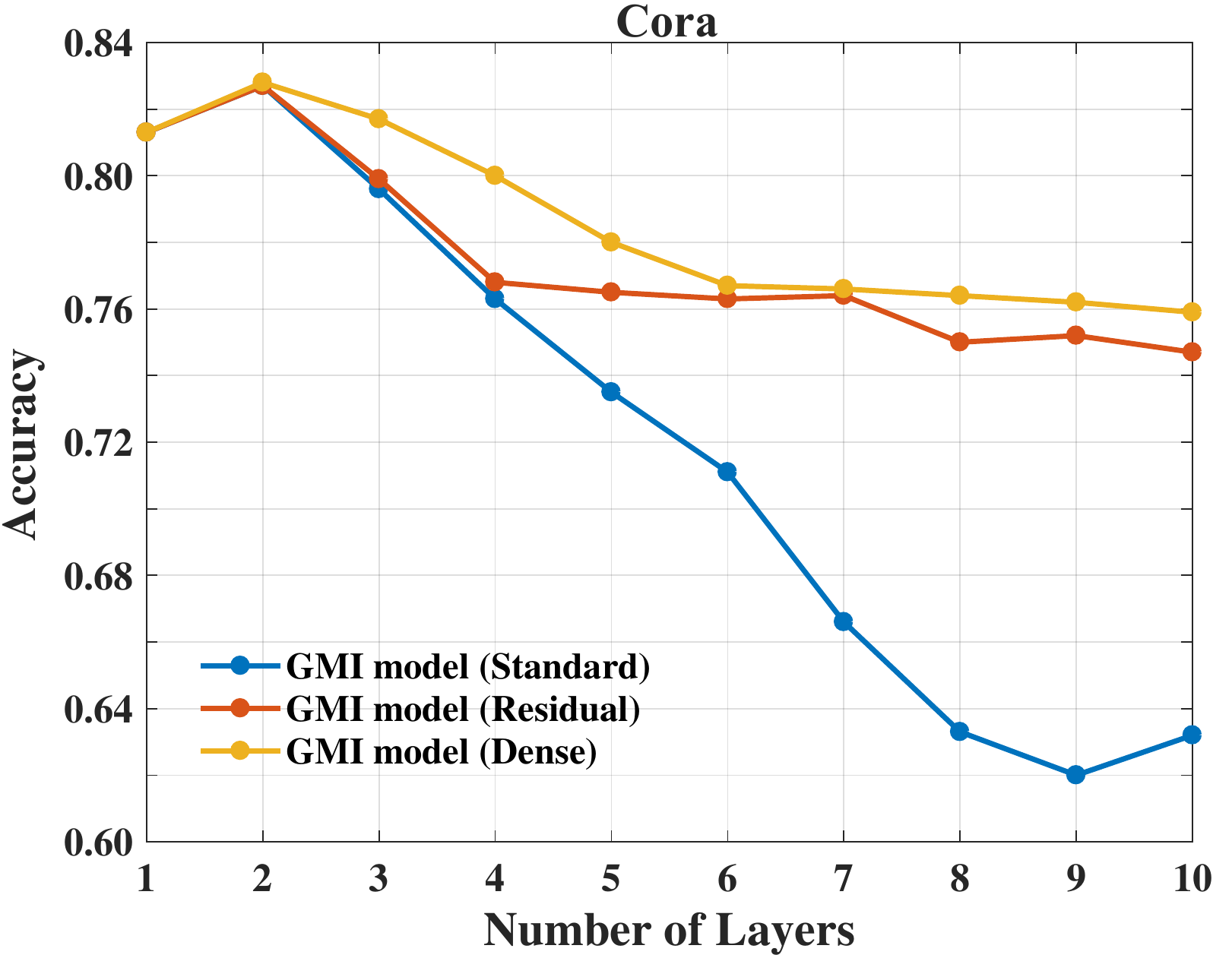}
	\end{minipage}
	\begin{minipage}[t]{0.4\linewidth}
		\centering
		\includegraphics[width=2.5in,height=1.78in]{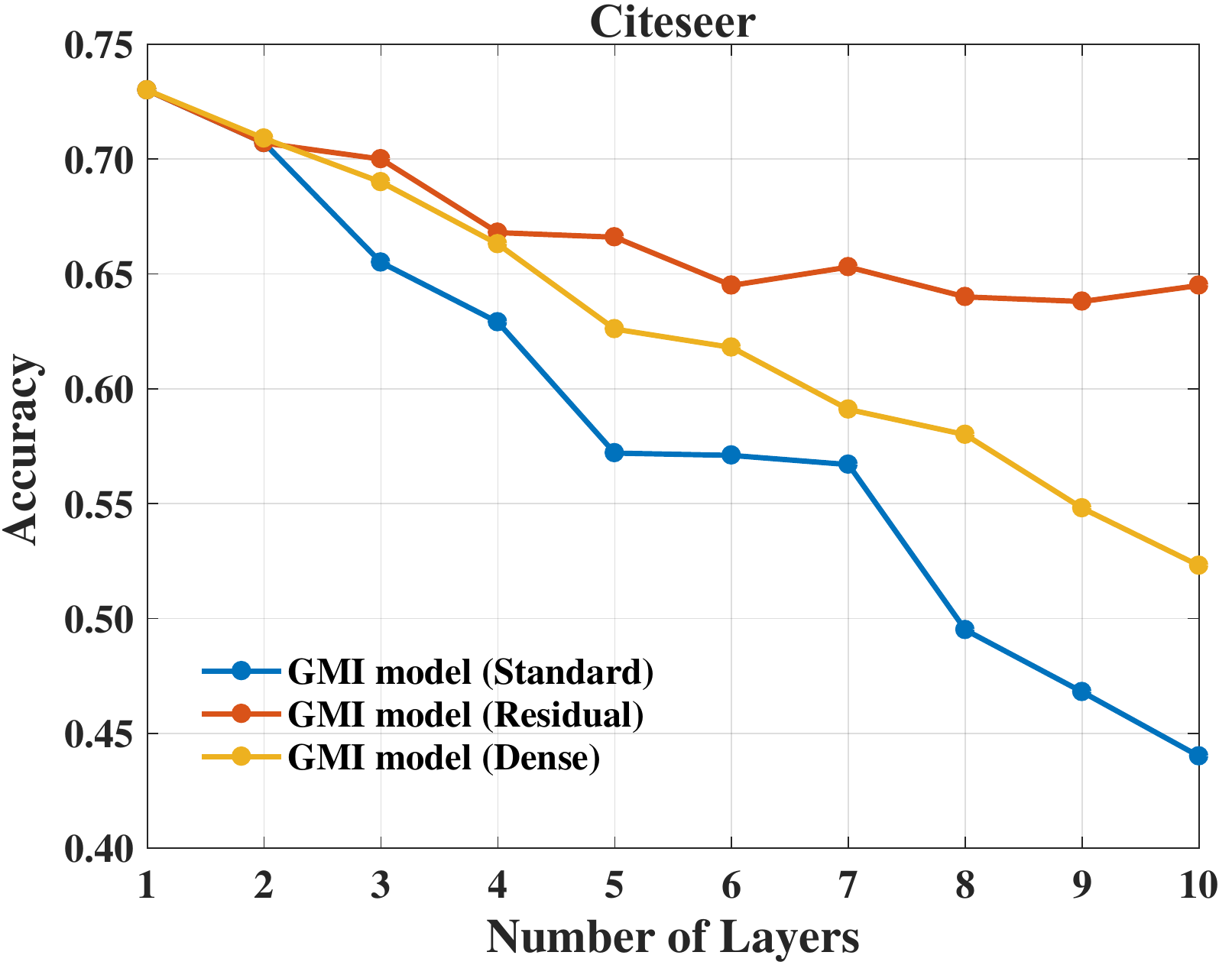}
	\end{minipage}
	\caption{Impact of model depth on classification accuracy.}
	\label{fig:graph1}
\end{figure*}

\subsection{Visualization}
For an intuitive illustration, Table~\ref{table:t6} displays t-SNE~\cite{maaten2008visualizing} plots of the learned embeddings on Cora and Citeseer. From a qualitative perspective, the distribution of plots learned by FMI and DGI seems to be similar, and the embeddings generated by GMI exhibit more discernible clusters than raw features, FMI, and DGI. Especially on Cora, the compactness and separability of clusters are extremely obvious, which represents the seven topic categories. As for quantitative analysis, we attempt to measure clustering quality by calculating the Silhouette Coefficient score~\cite{rousseeuw1987silhouettes}. Specifically, we employ \textit{silhouette\_score} function from the scikit-learn Python package~\cite{pedregosa2011scikit} with all default settings and follow the user guide to perform the evaluation. The clustering of embeddings learned via GMI obtains a Silhouette Coefficient score of 0.425 on Cora, 0.402 on Citeseer, and 0.385 on PubMed, while DGI gets 0.417, 0.391, 0.373 and EP-B gains 0.384, 0.385, 0.379 on the three datasets, respectively. Both qualitatively and quantitatively, it demonstrates the great performance of GMI, which illustrates the rationality and effectiveness of graphical mutual information maximization in unsupervised graph representation learning.

\subsection{Influence of Model Depth}
In this part, we adjust the number of convolutional layers in the encoder to investigate the influence of model depth on classification accuracy. Considering the potential difficulty of training deep neural networks, suggested by~\cite{he2016deep}, we also experiment with a counterpart residual version of the standard GMI model, which adds identity shortcut connections between every two hidden layers to improve the training of deep networks. Here, we continue to have $D'=512$ features for each hidden layer and start applying identity shortcuts from the second layer as the input and output of the first layer are not the same dimension. Moreover, compared to the standard GMI model that achieves GMI maximization between the final representation and original input graph, we consider another variant, called dense GMI, which maximizes GMI between each hidden layer and input graph. Figure~\ref{fig:graph3} gives a detailed architecture illustration. The involved hyperparameters remain unchanged except that we train for fixed epochs (600 on Cora and Citeseer) without early stopping. Results are plotted in Figure~\ref{fig:graph1}.

For one thing, the increase of model depth significantly widens the performance gap between models with and without shortcut connections. The best result for Cora is obtained with a two-layer GCN encoder, while the best result for Citeseer is achieved with a one-layer GCN encoder. Except for the fact that the increase of model depth makes training with no adoption of shortcut connections difficult, we also assume that the farther neighborhood information brought by multiple convolutional layers may be noise for self-representation learning. Specifically, the different proximity between neighbors means distinct extents of similarity, if two arbitrary nodes are a certain distance apart, they are likely to be completely different. Therefore, in the standard GMI model, the information aggregated from the farther neighborhood might contain much noise that is dissimilar to the characteristic of node itself, which degrades the quality of learned embeddings and subsequent classification performance. In contrast, additional identity shortcuts enable the model to carry over the information of the previous layer's input, which can be regarded as a complementary process to similar neighborhood information from shallower layers to deeper layers, thus the residual version is relatively less vulnerable to model depth. For another, we observe that the dense GMI variant can also alleviate the performance deterioration to some extent, although MI tends to decay with depth by data processing inequality~\cite{cover2012elements}. This thanks to maximizing graphical MI between the output of each layer and input graph, which imposes a direct constraint on each hidden layer to preserve input information as intact as possible. Based on this observation, enforcing the constraint of maximizing MI on hidden layers to reduce the loss of information when training deep neural networks could be a good practice.

\section{Conclusion}
To overcome the dilemma of lacking available supervision and evade the potential risk brought by unreliable labels, we introduce a novel concept of graphical mutual information (GMI) to carry out graph representation learning in an unsupervised pattern. Its core lies in directly maximizing the mutual information between the input and output of a graph neural encoder in terms of node features and topological structure. Through our theoretical analysis, we give a definition of GMI and decompose it into a form of a weighted sum which can be calculated by the current mutual information estimation method MINE easily. Accordingly, we develop an unsupervised model and conduct two common graph mining tasks. The results exhibit that GMI outperforms state-of-the-art unsupervised baselines across both classification tasks (transductive and inductive) and link prediction tasks, sometimes even be competitive with supervised algorithms. Future work will concentrate on task-oriented representation learning or adapting the idea of GMI maximization to other types of graphs such as heterogeneous graphs and hypergraphs.

\begin{acks}
	This work was supported by National Key Research and Development Program of China (No. 2018AAA0101400), National Nature Science Foundation of China (No. 61872287 and No. 61532015), Innovative Research Group of the National Natural Science Foundation of China (No. 61721002), Innovation Research Team of Ministry of Education (IRT\_17R86), and Project of China Knowledge Center for Engineering Science and Technology. Besides, this research was funded by National Science and Technology Major Project of the Ministry of Science and Technology of China (No. 2018AAA0102900). 
\end{acks}

\bibliographystyle{ACM-Reference-Format}
\balance
\bibliography{reference}


\end{document}